%% file: iclr2027_conference.tex
\documentclass{article} % For LaTeX2e
\usepackage{iclr2027_conference,times}
\iclrfinalcopy
\input{math_commands.tex}

\usepackage{hyperref}
\usepackage{url}

\usepackage{xcolor}

\usepackage{fontawesome5}
\definecolor{RefBlue}{HTML}{1F4E79}
\definecolor{CiteBlue}{HTML}{2A5DB0}
\definecolor{UrlCyan}{HTML}{007ACC}

\hypersetup{
    colorlinks=true,
    linkcolor=RefBlue,
    citecolor=CiteBlue,
    urlcolor=UrlCyan
}

\usepackage{algorithm}
\usepackage{algpseudocode}

\usepackage{booktabs}
\usepackage{amsfonts}
\usepackage{multirow}
\usepackage{graphicx}
\usepackage[table]{xcolor}
\usepackage{makecell}

\definecolor{posgreen}{RGB}{0,150,0}
\definecolor{negred}{RGB}{180,0,80}

\usepackage[table]{xcolor}

\definecolor{lightgraycell}{RGB}{242,242,242}
\definecolor{lightbluecell}{RGB}{235,245,255}

\definecolor{gold}{RGB}{212,175,55}
\definecolor{silver}{RGB}{160,160,160}
\definecolor{bronze}{RGB}{205,127,50}

\usepackage{newfloat}
\usepackage{listings}

\usepackage{booktabs}
\usepackage{amsmath} 

\usepackage{amsthm}
\usepackage{mathtools}
\usepackage{amssymb}

\usepackage{makecell}
\newtheorem{theorem}{Theorem}
\newtheorem{proposition}[theorem]{Proposition}
\newtheorem{lemma}[theorem]{Lemma}

\title{\textbf{\texttt{DreOPD}}: Degraded-Reference Extrapolative On-Policy Distillation for Flow-matching Models}

\author{
Mingfeng Lin\textsuperscript{1}\thanks{Equal contribution.}
\quad
Chengfei Cai\textsuperscript{2}\footnotemark[1]
\quad
Lin Xu\textsuperscript{1}
\quad
Yuxiang Wei\textsuperscript{3}
\quad
Liang Han\textsuperscript{1}\thanks{Corresponding author.}
\\
\textsuperscript{1}Harbin Institute of Technology (Shenzhen)
\\
\textsuperscript{2}Zhejiang University
\\
\textsuperscript{3}Harbin Institute of Technology
\\
\faGlobe\ \textbf{Project: }
\href{https://sleepy1231.github.io/DreOPD}
{\texttt{https://sleepy1231.github.io/DreOPD}}
}

\begin{document}

\maketitle

\vspace{-5pt}
\begin{figure}[!h]
\centering
\includegraphics[width=1.0\textwidth]{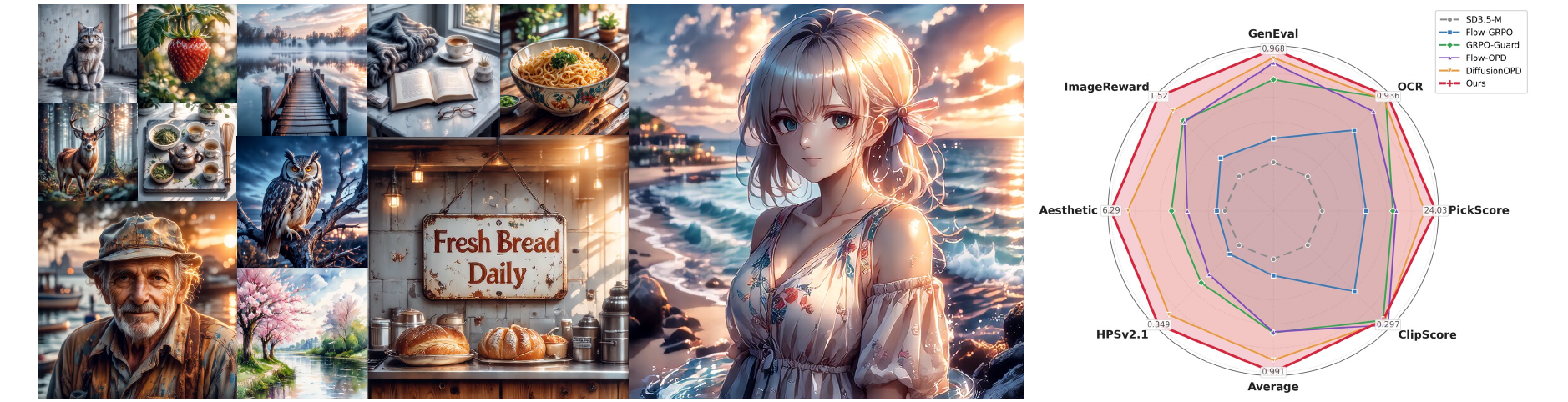}
\vspace{-5pt}
\caption{\textbf{Left:} Representative images generated by our method across diverse subjects, scenes, and text-rendering scenarios. \textbf{Right:} Our method achieves the highest average performance across the evaluated reward metrics.}
\label{fig:teaser}
\end{figure}

\begin{abstract}
Flow-matching models are now a mainstream method to image generation, but its adaptation to diverse downstream scenarios typically relies on post-training, which may cause conflicts among task-specific optimization objectives. Reinforcement learning enables direct optimization of task-specific rewards beyond the original models, yet trajectory-level optimization may incur high-variance gradients and cross-task interference. On-policy distillation (OPD) offers dense and stable supervision on student rollouts, but conventional teacher matching remains imitation-based. We propose \texttt{\textbf{DreOPD}}, a \textbf{D}egraded-\textbf{r}eference \textbf{e}xtrapolative \textbf{OPD} method for flow-matching models that bridges these two paradigms. Our \texttt{\textbf{DreOPD}} converts implicit reward extrapolation into closed-form velocity regression, enabling extrapolative post-training with the stability of OPD. It further uses a mildly degraded reference to strengthen the teacher-reference contrast, yielding a clearer extrapolation direction. Experiments on single- and multi-teacher settings show that \texttt{\textbf{DreOPD}} outperforms OPD and multi-task RL baselines in average performance, while surpassing specialized teachers on most metrics.
\end{abstract}

\section{Introduction}\label{sec:intro}

Flow-matching and diffusion models have emerged as leading paradigms for text-to-image generation~\citep{lipman2022flow,cao2025controllable}. Despite their strong general capabilities, practical deployment often requires post-training for diverse objectives, such as prompt following, text rendering, aesthetics, and human preference~\citep{li2026video,lyu2026flow}. These objectives may exhibit implicit conflicts, making it challenging for a single model to improve multiple capabilities simultaneously~\citep{yu2020gradient,hu2024harmodt}. Meanwhile, obtaining a strong task-specific teacher requires costly alignment, whereas weaker model variants are readily available through early training checkpoints or controlled degradation. Yet these weaker models are typically discarded rather than exploited as useful signals. This motivates a practical question: \emph{how can a student consolidate multiple strong teachers, leverage readily available weaker models, and improve beyond the teachers themselves?}

Existing approaches expose a fundamental trade-off. Reinforcement learning (RL)~\citep{black2024training,wallace2024diffusion,li2025mixgrpo,xue2025advantage,wang2026grpo,liu2026flow} directly optimize task rewards and can potentially improve beyond existing models. However, trajectory-level supervision may suffer from high variance, reward exploitation, and interference among heterogeneous objectives. On-policy distillation (OPD) instead provides dense and stable supervision by matching teachers on student-generated states~\citep{li2026diffusionopd,fang2026flowopd,zhou2026danceopd}. However, standard OPD is imitation-based: each teacher remains the pointwise target, causing a shared multi-task student to interpolate among specialized teachers rather than systematically exceed them. Thus, RL permits extrapolation but is difficult to optimize, whereas OPD is stable but teacher-bounded.

Reward extrapolation for autoregressive language models~\citep{yang2026gopd} offers a route between these two paradigms. It interprets the teacher-reference log-likelihood ratio as an implicit reward and amplifies this contrast beyond teacher matching. Directly transferring this principle to flow-matching models is nontrivial because such models predict continuous velocity fields rather than categorical token distributions~\citep{nie2026large}. A trajectory-level implementation would require likelihood-ratio estimation and credit assignment across continuous denoising steps~\citep{li2026diffusionopd}. The key challenge is therefore to translate distribution-level reward extrapolation into the native velocity-regression objective of flow matching.

We address this challenge with \texttt{\textbf{DreOPD}} (\textbf{D}egraded-\textbf{r}eference \textbf{e}xtrapolative \textbf{OPD}), a post-training method that converts implicit reward extrapolation into closed-form velocity regression for flow-matching models. Under shared-covariance Gaussian transitions~\citep{liu2026flow}, the conditional transition objective at each student-visited state has the pointwise optimizer $v_\lambda^\star=v_T+(\lambda-1)(v_T-v_{\mathrm{ref}})$. This target recovers teacher imitation at $\lambda=1$ and moves beyond the teacher when $\lambda>1$. It further reveals that the teacher-reference contrast determines the direction and magnitude of extrapolation. Motivated by this observation, we construct a mildly degraded reference that enlarges the contrast while preserving the generator's structure. As shown in Figure~\ref{fig:teaser}, \texttt{\textbf{DreOPD}} achieves the best average performance over prior methods and surpasses teachers on most metrics. Our contributions are summarized as follows:
\begin{itemize}
    \item We derive a closed-form target that extends flow-based OPD from teacher imitation to reward extrapolation.
    \item We characterize conditions under which the teacher-reference contrast is reward-aligned and introduce controlled reference degradation to strengthen this contrast.
    \item Across single- and multi-teacher settings, Our \texttt{\textbf{DreOPD}} achieves the best average performance while surpassing specialized teachers on most metrics.
\end{itemize}

\section{Related Work}
\subsection{Reinforcement Learning for Flow-matching Models}
Reinforcement learning has been increasingly applied to align diffusion and flow-based generative models with task-specific rewards. Early approaches optimize denoising policies through policy gradients or differentiate rewards through the sampling process, including DDPO~\citep{black2024training}, DPOK~\citep{fan2023dpok}, DRaFT~\citep{clark2024directly}, and AlignProp~\citep{prabhudesai2023aligning}. More recent methods, such as Flow-GRPO~\citep{liu2026flow} and DanceGRPO~\citep{xue2025dancegrpo}, adapt group-relative policy optimization~\citep{shao2024deepseekmath} to stochastic flow trajectories and assign terminal reward signals across continuous denoising steps. These methods are commonly instantiated to optimize a single model against one reward or a fixed aggregation of rewards. Extending them to multiple capabilities requires jointly optimizing heterogeneous objectives, whose gradients may interfere and induce trade-offs among tasks~\citep{li2025mixgrpo}. As a result, RL can be effective for specializing one capability, but multi-task RL remains difficult that the goal is to obtain a single model that performs well across diverse dimensions.

\subsection{On-Policy Distillation}
On-policy distillation~\citep{agarwal2024policy,li2026rethinking} trains a student on samples from its current policy, reducing the mismatch between the states encountered during training and generation. In language modeling, methods such as MiniLLM~\citep{gu2024minillm} and GKD~\citep{tan2023gkd} use student-generated sequences together with reverse-KL or generalized divergence objectives to transfer knowledge from a stronger teacher. Vision-OPD~\citep{yuan2026vision} extends OPD to multimodal LLMs, using a crop-conditioned teacher to supervise full-image student rollouts for fine-grained visual perception. G-OPD~\citep{yang2026gopd} interpret on-policy teacher matching as KL-regularized optimization under the implicit reward defined by the teacher-reference log-likelihood ratio, and generalize distillation by scaling this reward beyond exact teacher matching. For flow-matching models, Flow-OPD~\citep{fang2026flowopd} and DiffusionOPD~\citep{li2026diffusionopd} instead queries the teacher on states visited by student rollouts and regresses the student toward teacher velocity predictions. DanceOPD~\citep{zhou2026danceopd} further integrates text-to-image and image-to-image capabilities through OPD. Such regression provides dense local supervision, but standard teacher matching makes each teacher prediction the pointwise target~\citep{song2026survey}. In a multi-task setting, a shared student therefore tends to interpolate among task-specific teachers.

\section{Preliminaries}
\label{sec:preliminaries}

\paragraph{On-policy distillation and reward extrapolation.}

OPD trains a student distribution $\pi_\theta$ toward a teacher $\pi_T$ on samples generated by the student itself, typically by minimizing the reverse KL $D_\mathrm{KL}(\pi_\theta\|\pi_T)$. G-OPD~\citep{yang2026gopd} reinterprets teacher matching through the implicit reward $\log(\pi_T/\pi_{\mathrm{ref}})$ and introduces a scaling coefficient $\lambda$:
\begin{equation}
\mathbb{E}_{y\sim\pi_\theta}
\left[
\lambda
\log\frac{\pi_T(y)}{\pi_{\mathrm{ref}}(y)}
-
\log\frac{\pi_\theta(y)}{\pi_{\mathrm{ref}}(y)}
\right],
\label{eq:prelim-exopd}
\end{equation}
where $\pi_{\mathrm{ref}}$ is the reference. At $\lambda=1$, the reference terms cancel and the objective reduces to reverse-KL teacher matching. The regime $\lambda>1$, termed ExOPD, amplifies the teacher-reference contrast beyond teacher matching. This formulation establishes a distribution-level principle for reward extrapolation, but does not provide the local velocity target required by flow-matching training.

\section{Methodology}
\label{sec:method}
Reward extrapolation is naturally expressed for autoregressive language models~\citep{yang2026gopd}, which explicitly parameterize token distributions and provide token-level likelihood ratios. Flow-matching models~\citep{lipman2022flow}, however, parameterize velocity fields over continuous states, whereas the teacher-reference contrast is defined over complete trajectories. A direct trajectory-level implementation would therefore require likelihood-ratio estimation and long-horizon credit assignment over high-dimensional denoising transitions, making optimization potentially high-variance.

Under shared-covariance Gaussian transitions~\citep{song2020score}, we show that the conditional objective at each student-visited state admits the closed-form optimizer,
\begin{equation}
v_\lambda^\star
=
v_T + (\lambda-1)(v_T-v_{\mathrm{ref}}),
\label{eq:method-overview-target}
\end{equation}
where $v_T$, $v_{\mathrm{ref}}$, and $v_\lambda^\star$ denote the velocity fields of the teacher, the reference, and the extrapolated target. We train the student by regressing toward the target on its current rollouts. The teacher-reference contrast determines the direction and magnitude of extrapolation, motivating the degraded-reference construction introduced below. Figure~\ref{fig:concept} provides an overview.

\begin{figure}[h]
    \centering
    \includegraphics[width=1.0\textwidth]{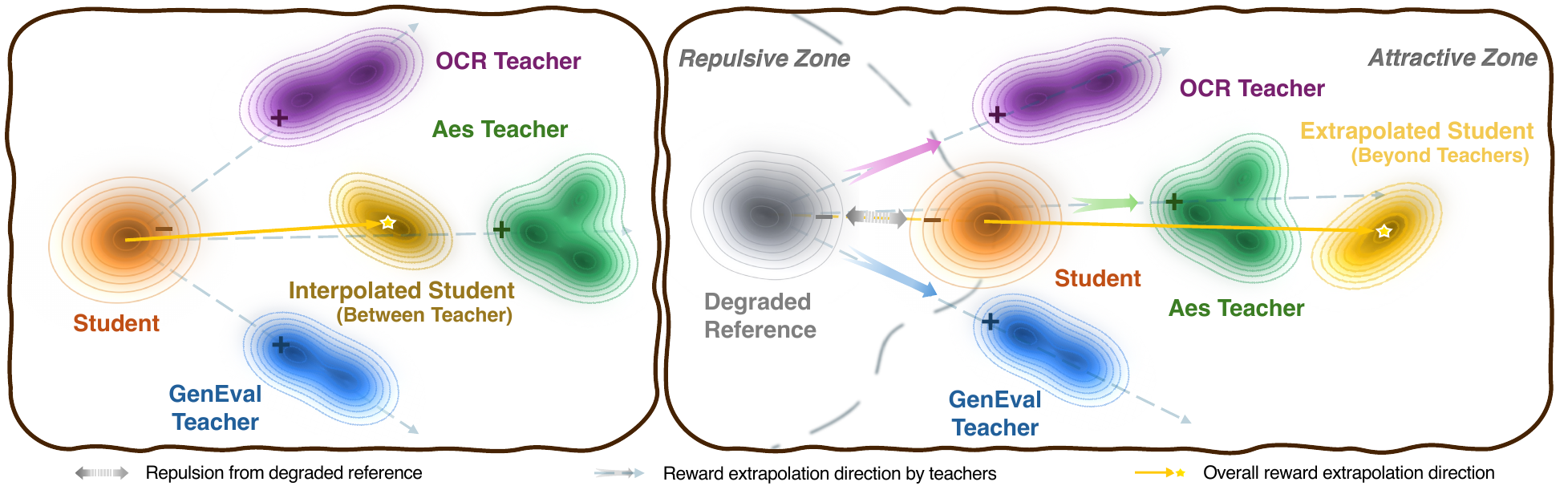}
    \caption{
    Conceptual comparison in the multi-task setting. \textbf{Left:} Standard OPD regresses a shared student toward task-specific teacher velocities, encouraging interpolation among the teachers. \textbf{Right:} \texttt{\textbf{DreOPD}} uses a shared degraded reference to construct targets that extrapolate through each task-specific teacher. Joint regression toward these targets moves the student beyond each teacher along the corresponding teacher-reference directions.
    }
    \label{fig:concept}
\end{figure}

\subsection{From Implicit Reward Extrapolation to Flow Velocity Targets}
\label{sec:closed-form-ire}

\paragraph{Trajectory-level design objective.}

Following the reverse-time sampling convention, we discretize the generation schedule as $
1=t_0>t_1>\cdots>t_N=0$ and $
\Delta t_j=t_{j+1}-t_j<0$. A generated trajectory is $
\tau =(x_{t_0},x_{t_1},\ldots,x_{t_N})$,
where $x_{t_0}$ follows the noise prior and $x_{t_N}$ is the generated sample. We distinguish the trajectory distribution $\Pi_\theta$ from its transition kernels $\pi_\theta^{(j)}$:
\begin{equation}
\Pi_\theta(\tau\mid c)
=
p(x_{t_0})
\prod_{j=0}^{N-1}
\pi_\theta^{(j)}
\left(
x_{t_{j+1}}\mid x_{t_j},c
\right),
\label{eq:trajectory-factorization}
\end{equation}
where $c$ denotes the conditioning input. We instantiate the distribution-level principle in Eq.~(\ref{eq:prelim-exopd}) over stochastic flow trajectories, obtaining the design objective
\begin{equation}
\begin{aligned}
\mathcal{J}(\theta)
={}&
\lambda\,
\mathbb{E}_{\tau\sim\Pi_\theta(\cdot\mid c)}
\left[
\log
\frac{\Pi_T(\tau\mid c)}
     {\Pi_{\mathrm{ref}}(\tau\mid c)}
\right]-
\mathrm{KL}
\left(
\Pi_\theta(\cdot\mid c)
\,\middle\|\,
\Pi_{\mathrm{ref}}(\cdot\mid c)
\right),
\end{aligned}
\label{eq:trajectory-objective}
\end{equation}
where $\lambda\geq 1$ controls the extrapolation strength. Equivalently,
\begin{equation}
\mathcal{J}(\theta)
=
\mathbb{E}_{\tau\sim\Pi_\theta}
\left[
\lambda
\log
\frac{\Pi_T(\tau\mid c)}
     {\Pi_{\mathrm{ref}}(\tau\mid c)}
-
\log
\frac{\Pi_\theta(\tau\mid c)}
     {\Pi_{\mathrm{ref}}(\tau\mid c)}
\right].
\label{eq:trajectory-objective-expanded}
\end{equation}

Directly optimizing Eq.~(\ref{eq:trajectory-objective-expanded}) over flow trajectories would require trajectory-level likelihood-ratio estimation and credit assignment across all denoising steps, which can introduce high-variance gradient~\citep{li2026diffusionopd}. Rather than estimating this trajectory-level policy gradient, we use the shared-covariance Gaussian transition structure to solve the conditional transition objective analytically, obtaining a closed-form velocity target at each student-visited state.

\paragraph{Gaussian flow transitions.}

Following Flow-GRPO~\citep{liu2026flow}, the Euler-Maruyama stochastic sampler uses Gaussian reverse-time transitions:
\begin{equation}
\pi_\theta^{(j)}
\left(
x_{t_{j+1}}\mid x_{t_j},c
\right)
=
\mathcal{N}
\left(
x_{t_{j+1}};
\mu_\theta^{(j)},
\Sigma_{t_j}
\right),
\label{eq:gaussian-transition}
\end{equation}
with $\mu_\theta=(1+\frac{\sigma_{t_j}^2\Delta t_j}{2t_j})x_{t_j}+(1+\frac{\sigma_{t_j}^2(1-t_j)}{2t_j})v_\theta \Delta t_j$ and $\Sigma_{t_j}=\sigma_{t_j}^2|\Delta t_j|\,I$. The transition kernels share $\Sigma_{t_j}$ and differ only in the means induced by their velocity fields.

Let $\mu_\theta^{(j)}$, $\mu_T^{(j)}$, and $\mu_{\mathrm{ref}}^{(j)}$ denote their transition means at a fixed visited state $(x_{t_j},t_j,c)$. For Gaussian distributions with shared covariance,
\begin{equation}
\begin{aligned}
\mathbb{E}_{\pi_\theta^{(j)}}
\left[
\log
\frac{\pi_T^{(j)}}
     {\pi_{\mathrm{ref}}^{(j)}}
\right]
={}&
\frac{
\left\|
\mu_\theta^{(j)}
-
\mu_{\mathrm{ref}}^{(j)}
\right\|^2
-
\left\|
\mu_\theta^{(j)}
-
\mu_T^{(j)}
\right\|^2
}
{2\sigma_{t_j}^2(-\Delta t_j)},
\\
\mathrm{KL}
\left(
\pi_\theta^{(j)}
\,\middle\|\,
\pi_{\mathrm{ref}}^{(j)}
\right)
={}&
\frac{
\left\|
\mu_\theta^{(j)}
-
\mu_{\mathrm{ref}}^{(j)}
\right\|^2
}
{2\sigma_{t_j}^2(-\Delta t_j)}.
\end{aligned}
\label{eq:gaussian-ratio-kl}
\end{equation}

The common prior $p(x_{t_0})$ cancels from all trajectory likelihood ratios. Using the chain factorization in Eq.~(\ref{eq:trajectory-factorization}), the trajectory objective can be written as a sum of conditional transition terms evaluated under student-induced state marginals. For a fixed visited state, maximizing the corresponding objective is equivalent to minimizing
\begin{equation}
\ell_j(\mu_\theta)
=\frac{
\lambda
\left\|
\mu_\theta^{(j)}-\mu_T^{(j)}
\right\|^2
-
(\lambda-1)
\left\|
\mu_\theta^{(j)}-\mu_{\mathrm{ref}}^{(j)}
\right\|^2
}{
2\sigma_{t_j}^2(-\Delta t_j)
}.
\label{eq:conditional-mean-loss}
\end{equation}

Since
$\mu_a^{(j)}-\mu_b^{(j)}
=
\Delta t_j
\left(
v_a-v_b
\right),$ the same conditional objective can be expressed in velocity space with a weight $\kappa_{t_j}$:
\begin{equation}
\ell_j(v_\theta)
=
\kappa_{t_j}[
\lambda
\left\|
v_\theta-v_T
\right\|^2
-
(\lambda-1)
\left\|
v_\theta-v_{\mathrm{ref}}
\right\|^2],\quad \kappa_{t_j}={-\Delta t_j}/{2\sigma_{t_j}^2}>0.
\label{eq:conditional-velocity-loss}
\end{equation}

Although Eq.~(\ref{eq:conditional-velocity-loss}) contains a negative quadratic term, its net coefficient on $\|v_\theta\|^2$ is $\lambda-(\lambda-1)=1.$ Therefore, the objective remains strongly convex in $v_\theta$ and admits a unique finite minimizer. Completing the square gives
\begin{equation}
\ell_j(v_\theta)
=\kappa_{t_j} 
\left\|
v_\theta-v_\lambda^\star
\right\|^2
+
C_j,
\label{eq:complete-square}
\end{equation}
where $C_j$ does not depend on $v_\theta$ and
\begin{equation}
v_\lambda^\star=
v_T
+
(\lambda-1)
\left(
v_T-v_{\mathrm{ref}}
\right).
\label{eq:extrapolated-velocity-target}
\end{equation}

Equation~(\ref{eq:extrapolated-velocity-target}) is the central target of \textbf{\texttt{DreOPD}}. At $\lambda=1$, it reduces to $v_\lambda^\star=v_T$, recovering teacher imitation. For $\lambda>1$, the target moves beyond the teacher along the teacher-reference direction. The displacement from the teacher is
\begin{equation}
v_\lambda^\star-v_T
=
(\lambda-1)
\left(
v_T-v_{\mathrm{ref}}
\right),
\label{eq:teacher-displacement}
\end{equation}
which makes the influence of both the extrapolation strength and the reference explicit.

\paragraph{On-policy regression objective.}

Given the extrapolated velocity target $v_\lambda^\star$, we optimize the target by on-policy regression objective:
\begin{equation}
\begin{aligned}
\mathcal{L}
(\theta)
=
\mathbb E_{c} \Bigg[
\sum_{j=0}^{N-1}
\kappa_{t_j}
\left(
\left\|
v_\theta(x_{t_j},t_j,c)
-
\mathrm{sg}(v_\lambda^\star(x_{t_j},t_j,c))
\right\|^2
\right)\Bigg],
\end{aligned}
\label{eq:on-policy-regression}
\end{equation}
where $\mathrm{sg}(\cdot)$ denotes stop-gradient. The trajectory $\{x_{t_j}\}_{j=0}^{N}$ is generated by the current student rollouts. Furthermore, under the deterministic ODE, the student, teacher and reference induce deterministic velocity predictions. We therefore optimize the same closed-form target using direct $\ell_2$ regression, weighted by ${(\Delta t_j)}^2/2$. Appendix~\ref{app:sde-to-ode} details the connection.

\subsection{Reward Interpretation of the Extrapolated Target}
\label{sec:reward-interpretation}

% The preceding derivation establishes a geometric result: the proposed target extrapolates from the reference through the teacher. Geometry alone does not imply improvement in a task reward. We next characterize when the teacher-reference direction admits a reward interpretation.

% The analysis is stated at the terminal-distribution level and characterizes when the teacher-reference density contrast is reward-aligned. It provides a reward interpretation for the contrast amplified by Eq.~(\ref{eq:extrapolated-velocity-target}).
The preceding derivation establishes a geometric result: the proposed target extrapolates from the reference through the teacher. This geometric property alone does not guarantee improvement under an arbitrary task reward. We therefore characterize a sufficient condition under which the teacher-reference contrast is reward-aligned. The analysis is stated at the terminal-distribution level and provides a reward interpretation for the contrast locally amplified by Eq.~(\ref{eq:extrapolated-velocity-target}).

% \paragraph{Teacher as a reward-tilted reference.}

\begin{lemma}[Teacher as a reward-tilted reference]
\label{lem:teacher-reward-tilt}
Let $r(x)$ be a bounded reward and $\beta>0$ a temperature. Suppose the teacher solves the KL-regularized reward optimization problem $p_T=\arg\max_p \left\{\mathbb{E}_{x\sim p}[r(x)]-\beta\,\mathrm{KL}\left(p\,\middle\|\,p_{\mathrm{ref}}\right)\right\}$. Then $ p_T(x) = ({1}/{Z_T})\ p_{\mathrm{ref}}(x)
\exp
\left({r(x)}/{\beta}
\right),
$ where $ Z_T = \int
p_{\mathrm{ref}}(x)
\exp \left({r(x)}/{\beta}\right)
dx$ is a normalizing constant.
\end{lemma}

Lemma~\ref{lem:teacher-reward-tilt} implies $\log[p_T(x)/p_{\mathrm{ref}}(x)]={r(x)}/{\beta}-\log Z_T.$ Under this assumption, the teacher-reference log-density ratio recovers the reward $r(x)$ up to a positive scale and an additive constant.

% Motivated by the geometric form of
% Eq.~(\ref{eq:extrapolated-velocity-target}), consider the distribution-level extrapolation
% \begin{equation}
%  
% p_\lambda(x)
% =
% \frac{1}{Z_\lambda}
% p_T(x)^\lambda
% p_{\mathrm{ref}}(x)^{1-\lambda},
% \qquad
% \lambda\geq 1.
% \label{eq:distribution-extrapolation}
% \end{equation}

% Its log-density satisfies $\log p_\lambda(x)=\log p_T(x)+(\lambda-1)\left[\log p_T(x)-\log p_{\mathrm{ref}}(x)\right]-\log Z_\lambda.$
% Thus, $p_\lambda$ amplifies the same teacher-reference log-density direction that appears in Eq.~(\ref{eq:ratio-recovers-reward}).

% The terminal-distribution counterpart of the objective also admits a closed-form optimizer, clarifying how $\lambda$ modifies the teacher distribution.

\begin{proposition}[Distributional optimum and reward monotonicity]
\label{prop:reward-monotonicity}
Consider the distribution-level objective
\begin{equation}
\mathcal{J}_\lambda(p)
=
\lambda\,
\mathbb{E}_{x\sim p}
\left[
\log\frac{p_T(x)}{p_{\mathrm{ref}}(x)}
\right]
-
\mathrm{KL}\left(p\,\middle\|\,p_{\mathrm{ref}}\right),
\end{equation}
where $\lambda\geq1$. Assuming the normalizing constant is finite, its
optimizer over distributions absolutely continuous with respect to
$p_{\mathrm{ref}}$ is
\begin{equation}
p_\lambda(x)
=
\frac{1}{Z_\lambda}
p_T(x)^\lambda
p_{\mathrm{ref}}(x)^{1-\lambda}.
\label{eq:distribution-extrapolation}
\end{equation}

Moreover, under the reward-tilt assumption of
Lemma~\ref{lem:teacher-reward-tilt}, the expected reward
$J(\lambda)=\mathbb{E}_{x\sim p_\lambda}[r(x)]$ satisfies
\begin{equation}
\frac{dJ(\lambda)}{d\lambda}
=
\frac{1}{\beta}
\mathrm{Var}_{x\sim p_\lambda}[r(x)]
\geq 0,
\label{eq:reward-monotonicity}
\end{equation}
with strict inequality whenever $r$ is non-constant
$p_\lambda$-almost surely.
\end{proposition}

Proofs are provided in Appendix~\ref{app:proof-reward-tilt} and Appendix~\ref{app:proof-reward-monotonicity}. The optimal distribution satisfies $\log p_\lambda=\log p_T+(\lambda-1)\left(\log p_T-\log p_{\mathrm{ref}}\right)-\log Z_\lambda$. Thus, $\lambda>1$ extrapolates beyond the teacher along the teacher-reference log-density contrast. Under the reward-tilt assumption, this contrast is proportional to the underlying reward up to scale and a constant, yielding the monotonicity result in Eq.~(\ref{eq:reward-monotonicity}). Proposition~\ref{prop:reward-monotonicity} establishes the reward interpretation of extrapolation at the distribution level, and our Gaussian transition derivation translates this principle into the extrapolative velocity target used by our \textbf{\texttt{DreOPD}}.

\subsection{Amplifying Extrapolation with a Degraded Reference}
\label{sec:degraded-reference}

Eq.~(\ref{eq:teacher-displacement}) shows that the displacement beyond the teacher scales with $v_T-v_{\mathrm{ref}}$. Reference selection therefore determines the direction and magnitude of the extrapolation.

\paragraph{Reference determines the extrapolation contrast.}

In standard KL-regularized optimization, the reference acts primarily as a conservative anchor that discourages large policy changes. In Eq.~(\ref{eq:extrapolated-velocity-target}), it has an additional geometric role. The teacher attracts the student, while the reference specifies the direction from which the target moves through and beyond the teacher. A reference close to the teacher yields a small contrast, making the extrapolation signal weak. This motivates using a mildly degraded reference to enlarge teacher-reference contrast while retains reward-aligned directions.

\paragraph{Reward-alignment model.}

We formalize this intuition through an idealized family of reward-aligned distributions. Let $p_0$ denote the distribution induced by the pretrained generator that initializes the student, and define
\begin{equation}
p_\rho(x)
=
\frac{1}{Z_\rho}
p_0(x)
\exp
\left(
\frac{\rho\,r(x)}{\beta}
\right),
\qquad
\rho\leq 1,
\label{eq:reward-alignment-family}
\end{equation}
where $\rho$ denotes the level of alignment with reward $r$. We parameterize the teacher as $p_T=p_1$ and the pretrained model as $p_0$. A degraded reference with lower reward alignment than the pretrained model is represented by $\rho<0$.

\begin{proposition}[Reference degradation amplifies extrapolation]
\label{thm:reference-amplification}
Let $p_T=p_1$ and $p_{\mathrm{ref}}=p_\rho$ be members of the family in Eq.~(\ref{eq:reward-alignment-family}). Then the extrapolated distribution $ p_{\lambda,\rho}(x) \propto p_T(x)^\lambda p_\rho(x)^{1-\lambda} $ is also a reward tilt of $p_0$:
\begin{equation}
p_{\lambda,\rho}(x)
=
\frac{1}{Z_{\lambda,\rho}}
p_0(x)
\exp
\left(
\frac{
\rho_{\mathrm{eff}}(\lambda,\rho)
r(x)}
{\beta}
\right),
\label{eq:effective-alignment-distribution}
\end{equation}
where $\rho_{\mathrm{eff}}(\lambda,\rho)=1+(\lambda-1)(1-\rho)$. Therefore, for a degraded reference $\rho_d$ and a normal reference $\rho_n$ satisfying $\rho_d<\rho_n<1$, any fixed $\lambda>1$ yields
\begin{equation}
\rho_{\mathrm{eff}}(\lambda,\rho_d)
>
\rho_{\mathrm{eff}}(\lambda,\rho_n)> 1.
\label{eq:larger-extrapolation-step}
\end{equation}

Moreover, if $r$ is nonconstant under $p_T$, their local reward sensitivities at $\lambda=1$ satisfy
\begin{equation}
\left.
\frac{
dJ_{\rho_d}(\lambda)/d\lambda
}{
dJ_{\rho_n}(\lambda)/d\lambda
}
\right|_{\lambda=1}
=
\frac{1-\rho_d}{1-\rho_n}
>
1,\
J_\rho(\lambda)
=
\mathbb{E}_{p_{\lambda,\rho}}[r(x)].
\label{eq:reference-amplification-ratio}
\end{equation}
\end{proposition}

Proofs are in Appendix~\ref{app:proof-reference-amplification}. At $\lambda=1$, every reference choice recovers the teacher distribution, whereas the derivative with respect to $\lambda$ scales with the contrast magnitude $1-\rho$. Proposition~\ref{thm:reference-amplification} therefore shows that, within the reward-alignment model, a lower-alignment reference increases both the effective reward-tilt coefficient and the local rate of reward change without altering the reward direction. This result motivates controlled reference degradation as a mechanism for strengthening extrapolation.

\paragraph{Practical degraded-reference construction.}

In practice, we construct the degraded reference from the same base model using controlled corruption such as weight quantization or mild noise injection into the velocity output. The degradation level must balance contrast strength against structural preservation. Insufficient degradation produces a weak teacher-reference contrast, whereas excessive or poorly structured degradation may cause 
$v_T-v_{\mathrm{ref}}$ to capture corruption artifacts rather than task-relevant differences. We therefore evaluate multiple degradation mechanisms and strengths in the following section.

\subsection{Training procedure}
\label{sec:training-algorithm}

Algorithm~\ref{alg:dr-ire} summarizes the training procedure of \texttt{\textbf{DreOPD}}. Given task-specific teachers trained by existing RL algorithms~\citep{zheng2025diffusionnft,wang2026grpo,zhao2026marble} and a degraded reference, each training round iterates over the tasks. For each task, the current student generates on-policy trajectories and the corresponding teacher and reference are queried at the visited states. Finally, the student is updated toward the closed-form target. 

\begin{algorithm}[t]
\caption{\texttt{\textbf{DreOPD}}}
\label{alg:dr-ire}
\begin{algorithmic}[1]
\Require teachers $\{v_T^{(m)}\}_{m=1}^{M}$; degraded reference $v_{\mathrm{deg}}$; student $v_\theta$; extrapolation strength $\lambda>1$; tasks $\mathcal{M}=\{1,\ldots,M\}$; prompt datasets $\{\mathcal{C}^{(m)}\}_{m=1}^{M}$
\For{each training round}
    \For{$m=1,\ldots,M$}
        \State Sample prompts $c\sim\mathcal{C}^{(m)}$
        \State Rollout the student $v_\theta$ on $c$ to obtain trajectory $\{x_{t_j}\}_{j=0}^{N}$
        \State Query $v_T^{(m)}(x_{t_j},t_j,c)$ and $v_{\mathrm{deg}}(x_{t_j},t_j,c)$ for $j\in\{0,\ldots,N-1\}$ \Comment{no grad}
        \State Compute $\mathcal{L}_m(\theta)$ in Eq.~(\ref{eq:on-policy-regression})
        \State Update $\theta$ by one optimizer step on $\mathcal{L}_m(\theta)$
    \EndFor
\EndFor
\end{algorithmic}
\end{algorithm}

\section{Experiments}\label{sec:exp}
% We evaluate whether the proposed method can improve beyond single-task RL teachers and whether it can integrate multiple reward-specialized teachers into a unified model. We consider three representative text-to-image objectives: compositional correctness measured by GenEval, text rendering measured by OCR, and aesthetic quality measured by PickScore, CLIPScore, and HPSv2.1.

\subsection{Experimental Setup}
\paragraph{Implementation details.}
We conduct all experiments with SD3.5-M~\citep{esser2024sd} at a resolution of $512\times512$, keeping the teacher and reference models frozen throughout distillation. GenEval~\citep{ghosh2023geneval} and OCR use the data splits released with FlowGRPO~\citep{liu2026flow}, while aesthetic optimization is performed on Pick-a-Pic~\citep{kirstain2023pick} and evaluated on DrawBench~\citep{saharia2022drawbench}. We also report Aesthetic~\citep{schuhmann2022aesthetics} and ImageReward~\citep{xu2023imagereward} as out-of-domain reward. For comparability with prior work, we follow the configuration of DiffusionOPD, with additional details provided in Appendix~\ref{app:experimental-details}.

\begin{table*}[!t]
\centering
{
\small
\setlength{\tabcolsep}{3.5pt}
\renewcommand{\arraystretch}{0.8}
\begin{tabular}{l|c|c|ccc}
\toprule
\multirow{2}{*}[-0.5ex]{\textbf{Method}}
& \multicolumn{1}{c|}{\textbf{GenEval Student}}
& \multicolumn{1}{c|}{\textbf{OCR Student}}
& \multicolumn{3}{c}{\textbf{Aesthetic Student}} \\
\cmidrule(lr){2-2}
\cmidrule(lr){3-3}
\cmidrule(lr){4-6}
& \textbf{GenEval}
& \textbf{OCR}
& \textbf{PickScore}
& \textbf{ClipScore}
& \textbf{HPSv2.1}\\
\midrule
\textbf{\texttt{SD3.5-M (w/o CFG)}}
& 0.2529 & 0.1377 & 20.519 & 0.2384 & 0.2052 \\
\textbf{\texttt{SD3.5-M}}
& 0.6273 & 0.5079 & 22.331 & 0.2837 & 0.2795 \\
\textbf{\texttt{Teacher}}
& 0.9470 & 0.9239 & \underline{24.034} & \textbf{0.2963} & 0.3460 \\
\midrule
\textbf{\texttt{Flow-OPD}}
& 0.9501$_{\textcolor{posgreen}{+.0031}}$ & 0.9279$_{\textcolor{posgreen}{+.0040}}$ & 24.006$_{\textcolor{negred}{-.0280}}$ & \underline{0.2957}$_{\textcolor{negred}{-.0006}}$ & 0.3445$_{\textcolor{negred}{-.0015}}$ \\
\textbf{\texttt{DiffusionOPD}}
& 0.9648$_{\textcolor{posgreen}{+.0178}}$ & 0.9246$_{\textcolor{posgreen}{+.0007}}$ & 24.008$_{\textcolor{negred}{-.0260}}$ & 0.2955$_{\textcolor{negred}{-.0008}}$ & 0.3454$_{\textcolor{negred}{-.0006}}$ \\
\rowcolor{lightbluecell}
\textbf{\texttt{Ours (w/o DeRef.)}}
& \underline{0.9708}$_{\textcolor{posgreen}{+.0238}}$ & \underline{0.9322}$_{\textcolor{posgreen}{+.0083}}$ & \underline{24.034}$_{\textcolor{posgreen}{+.0000}}$ & 0.2946$_{\textcolor{negred}{-.0017}}$ & \underline{0.3492}$_{\textcolor{posgreen}{+.0032}}$ \\
\rowcolor{lightbluecell}
\textbf{\texttt{Ours}}
& \textbf{0.9710}$_{\textcolor{posgreen}{+.0240}}$ & \textbf{0.9364}$_{\textcolor{posgreen}{+.0125}}$ & \textbf{24.037}$_{\textcolor{posgreen}{+.0030}}$ & 0.2955$_{\textcolor{negred}{-.0008}}$ & \textbf{0.3497}$_{\textcolor{posgreen}{+.0037}}$ \\
\bottomrule
\end{tabular}
}
\caption{Single-teacher distillation on three independent tasks. GenEval and OCR are reported by their respective students, while PickScore, ClipScore, and HPSv2.1 are reported by the aesthetics student. The \textbf{\texttt{Teacher}} row similarly combines the in-domain scores of three task-specific teachers. \textbf{Bold}: best; \underline{Underline}: second best. Subscripts show absolute changes from the corresponding teacher. \textbf{DeRef.} denotes the degraded reference.}
\label{tab:main_results_single}
\end{table*}

\paragraph{Single-task RL teachers.}
We use the same specialized teachers as DiffusionOPD~\citep{li2026diffusionopd}. The GenEval teacher is trained with DiffusionNFT~\citep{zheng2025diffusionnft}, while the OCR and aesthetics teachers are trained with GRPO-Guard~\citep{wang2026grpo}. The aesthetics teacher jointly optimizes an equally weighted combination of PickScore, ClipScore~\citep{hessel2021clipscore}, and HPSv2.1~\citep{wu2023hps}. Each teacher is optimized independently on its corresponding task, providing a strong task-specific target for subsequent distillation.

\paragraph{Multi-task RL baselines.}
We compare against Flow-GRPO~\citep{liu2026flow}, GRPO-Guard~\citep{wang2026grpo}, and DiffusionNFT~\citep{zheng2025diffusionnft} trained directly in the multi-task setting. These methods optimize a single model over the three tasks by alternating among their corresponding datasets. We additionally include CascadeNFT, which optimizes the tasks sequentially. These baselines represent direct approaches to obtaining a generalist model.

\paragraph{OPD baselines.}
For a fair comparison, all OPD methods use the same task-specific teachers. Flow-OPD~\citep{fang2026flowopd} performs teacher-oriented velocity regression and uses Manifold Anchor Regularization to constrain the student to a high-quality visual manifold. DiffusionOPD~\citep{li2026diffusionopd} derives a closed-form reverse-KL objective for consolidating specialized teachers. In contrast, \textbf{\texttt{DreOPD}} explicitly constructs extrapolated velocity targets from the teacher-reference contrast.

% Our method instead regresses toward the extrapolated target $v_\lambda^\star=v_T+(\lambda-1)(v_T-v_{\mathrm{ref}})$, allowing the student to move beyond rather than merely match each teacher. We report our method both with the original reference and with the proposed degraded reference (\textbf{DeRef.}).

\subsection{Single-Teacher Distillation}
\label{sec:single-teacher}

We investigate whether on-policy distillation can surpass specialized teachers within their respective domains. Table~\ref{tab:main_results_single} summarizes three independent experiments: separate students are distilled for GenEval, OCR, and aesthetics, with the aesthetics student evaluated by PickScore, ClipScore, and HPSv2.1. Accordingly, the \textbf{\texttt{Teacher}} row combines the in-domain scores of the three corresponding task-specific teachers rather than representing a single teacher model.

The conventional OPD baselines remain close to their corresponding teachers. Flow-OPD yields small gains on GenEval and OCR but slightly decreases all three aesthetics metrics. DiffusionOPD improves GenEval more substantially, while remaining close to the OCR teacher and slightly underperforming the aesthetics teacher on its three evaluation metrics. These results are consistent with objectives centered primarily on teacher matching.

In contrast, our method improves GenEval from $0.9470$ to $0.9710$, OCR from $0.9239$ to $0.9364$, PickScore from $24.034$ to $24.037$ and HPSv2.1 from $0.3460$ to $0.3497$. Overall, it surpasses the corresponding teachers on four of five metrics. Its consistent gain over the non-degraded variant further demonstrates that the degraded reference strengthens extrapolation beyond the teachers.

\begin{table*}[!t]
\centering
\setlength{\tabcolsep}{1mm}
{
\scriptsize
\renewcommand{\arraystretch}{0.9}
\begin{tabular}{l|ccccc|cc|c}
\toprule
\textbf{Model} & \textbf{GenEval} & \textbf{OCR} & \textbf{PickScore} & \textbf{ClipScore} & \textbf{HPSv2.1} & \textbf{Aesthetic} & \textbf{ImgReward} & \textbf{Avg.} \\
\midrule
\textbf{\texttt{SD3.5-M$^\dagger$}}
& 0.2529 & 0.1377 & 20.519 & 0.2384 & 0.2052 & 5.161 & -0.5471 & 0.0000 \\
\textbf{\texttt{SD3.5-M}}
& 0.6273 & 0.5079 & 22.331 & 0.2837 & 0.2795 & 5.396 & 0.8324 & 0.5219 \\
\midrule
\multicolumn{7}{l}{\textit{Single-Task RL Teachers}} \\
\textbf{\texttt{GenEval}}
& \cellcolor{lightgraycell}{0.9470} & 0.6286 & 20.084 & 0.2870 & 0.2644 & 5.246 & 0.8976 & 0.5755 \\
\textbf{\texttt{OCR}}
& 0.6562 & \cellcolor{lightgraycell}{0.9239} & 22.225 & 0.2919 & 0.2720 & 5.266 & 0.8881 & 0.5963 \\
\textbf{\texttt{Aes}}
& 0.4935 & 0.5014 & \cellcolor{lightgraycell}{\underline{24.034}} & \cellcolor{lightgraycell}{\underline{0.2963}} & \cellcolor{lightgraycell}{\underline{0.3460}} & 6.232 & 1.5071 & 0.8120 \\
\midrule
\multicolumn{7}{l}{\textit{Multi-Task RL}} \\
\textbf{\texttt{Flow-GRPO}}
& 0.7399 & 0.7673 & 22.677 & 0.2893 & 0.2901 & 5.469 & 1.0539 & 0.6527 \\
\textbf{\texttt{GRPO-Guard}}
& 0.9002 & 0.9278 & 23.197 & 0.2959 & 0.3137 & 5.824 & 1.3582 & 0.8394 \\
\textbf{\texttt{DiffNFT}}
& 0.9510 & \textbf{0.9491} & 23.182 & 0.2869 & 0.2812 & 5.399 & 1.0699 & 0.6939 \\
\textbf{\texttt{CascadeNFT}}
& 0.9376 & 0.8827 & 23.803 & 0.2920 & 0.3305 & 6.006 & 1.4907 & 0.9029 \\
\midrule
\multicolumn{7}{l}{\textit{On-Policy Distillation}} \\
\textbf{\texttt{Flow-OPD}}
& 0.9395$_{\textcolor{negred}{-.0075}}$ & 0.8756$_{\textcolor{negred}{-.0048}}$ & 23.262$_{\textcolor{negred}{-.7720}}$ & \textbf{0.2975}$_{\textcolor{posgreen}{+.0012}}$ & 0.3070$_{\textcolor{negred}{-.0039}}$ & 5.692$_{\textcolor{negred}{-0.540}}$ & 1.3251$_{\textcolor{negred}{-.1820}}$ & 0.8189 \\
\textbf{\texttt{DiffOPD}}
& 0.9607$_{\textcolor{posgreen}{+.0137}}$ & 0.9242$_{\textcolor{posgreen}{+.0003}}$ & 23.980$_{\textcolor{negred}{-.0540}}$ & 0.2962$_{\textcolor{negred}{-.0001}}$ & 0.3422$_{\textcolor{negred}{-.0038}}$ & 6.191$_{\textcolor{negred}{-0.041}}$ & 1.5017$_{\textcolor{negred}{-.0054}}$ & 0.9680 \\
\rowcolor{lightbluecell}
\textbf{\texttt{Ours$^\dagger$}}
& \underline{0.9668}$_{\textcolor{posgreen}{+.0198}}$ & 0.9281$_{\textcolor{posgreen}{+.0042}}$ & 24.032$_{\textcolor{negred}{-.0020}}$ & 0.2956$_{\textcolor{negred}{-.0007}}$ & \textbf{0.3487}$_{\textcolor{posgreen}{+.0027}}$ & \underline{6.239}$_{\textcolor{posgreen}{+0.007}}$ & \underline{1.5172}$_{\textcolor{posgreen}{+.0001}}$ & \underline{0.9841} \\
\rowcolor{lightbluecell}
\textbf{\texttt{Ours}}
& \textbf{0.9681}$_{\textcolor{posgreen}{+.0211}}$ & \underline{0.9362}$_{\textcolor{posgreen}{+.0123}}$ & \textbf{24.035}$_{\textcolor{posgreen}{+.0010}}$ & 0.2959$_{\textcolor{negred}{-.0004}}$ & \textbf{0.3487}$_{\textcolor{posgreen}{+.0027}}$ & \textbf{6.292}$_{\textcolor{posgreen}{+0.060}}$ & \textbf{1.5245}$_{\textcolor{posgreen}{+.0174}}$ & \textbf{0.9939} \\
\bottomrule
\end{tabular}
}
\caption{Results of single-task RL, multi-task RL and OPD methods. \textbf{\texttt{SD3.5-M}}$^\dagger$: without CFG; \textbf{\texttt{Ours}}$^\dagger$: without degraded reference. \textbf{Bold}: best; \underline{Underline}: second-best; \colorbox{lightbluecell}{Blue-colored}: Ours; \colorbox{lightgraycell}{Gray-colored}: In-Domain reward; Average (Avg.): the mean of the min-max normalized scores across all metrics. Subscripts show absolute changes from the corresponding teacher. Out-Of-Domain rewards: Aesthetic and ImgReward.}
\label{tab:main_results}
\end{table*}

\subsection{Multi-Teacher Distillation}

Table~\ref{tab:main_results} compares task-specific teachers, multi-task RL, and multi-teacher OPD. Each teacher excels in its own domain but transfers poorly to other objectives. Each teacher excels in its own domain but transfers poorly across objectives, motivating their consolidation into a single generalist model.

Multi-task RL improves task coverage but suffers from optimization interference. DiffusionNFT achieves the highest OCR score at the expense of perceptual quality, while GRPO-Guard is more balanced but remains suboptimal overall. CascadeNFT performs better through sequential optimization but still falls short of the strongest distillation methods. These results reveal the task conflict induced by jointly optimizing heterogeneous rewards.

Multi-teacher OPD instead separates teacher specialization from knowledge consolidation. Flow-OPD underperforms the teachers on several metrics, while DiffusionOPD largely preserves their capabilities. Our \textbf{\texttt{DreOPD}} achieves the highest average score of $0.9939$, outperforming the corresponding teachers on GenEval, OCR, HPSv2.1, Aesthetic and ImageReward metrics while matching PickScore and retaining comparable ClipScore. 

The degraded reference further raises the average score from $0.9841$ to $0.9939$, with the largest gains on OCR and aesthetic metrics. This supports our hypothesis that enlarging a structurally meaningful teacher-reference contrast strengthens extrapolation. Figure~\ref{fig:qualitative} shows qualitative results.

\begin{figure*}[!t]
\centering
\includegraphics[width=1.0\textwidth]{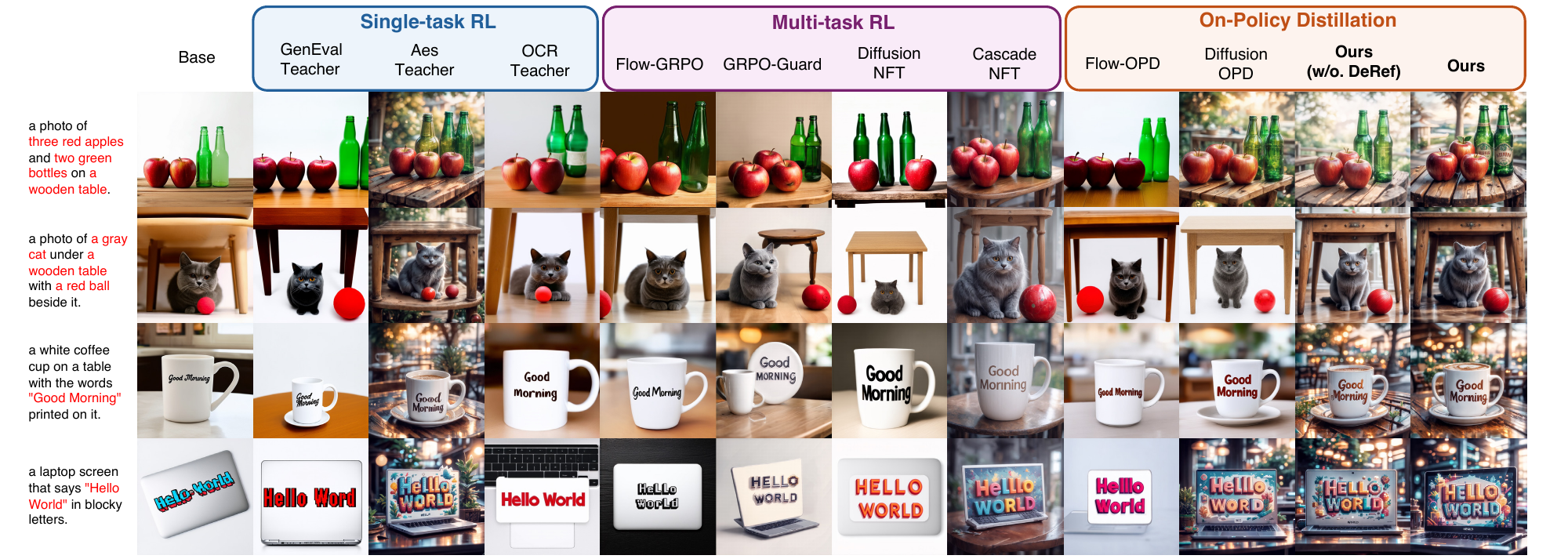}
\caption{Qualitative comparison of the base model, single-task teachers, multi-task RL, OPD baselines, and our method.}
\label{fig:qualitative}
\end{figure*}

\subsection{Ablation Studies}\label{sec:ablations}

\paragraph{Impact of factor $\lambda$.}

Table~\ref{tab:lambda} studies the effect of $\lambda$ using the original reference. Interpolation ($\lambda<1$) generally underperforms teacher matching ($\lambda=1$), whereas extrapolation ($\lambda>1$) performs better overall. We use $\lambda=1.25$ which provides the best cross-task balance. Although $\lambda=1.5$ further improves OCR, it degrades GenEval and perceptual quality, suggesting that excessive extrapolation may over-emphasize one capability at the cost of overall alignment.

\begin{table}[!h]
\centering
{
\small
\renewcommand{\arraystretch}{0.95}

\begin{tabular}{@{}lccccccc@{}}
\toprule
\multirow{2}{*}[-0.5ex]{\textbf{$\lambda$}}
& \multicolumn{5}{c}{\textbf{In-Domain}}
& \multicolumn{2}{c}{\textbf{OOD}}\\
\cmidrule(lr){2-6} \cmidrule(lr){7-8}
& \textbf{GE}
& \textbf{OCR}
& \textbf{Pick}
& \textbf{Clip}
& \textbf{HPS}
& \textbf{Aes}
& \textbf{IR} \\
\midrule
0.5
& 0.921 & 0.899 & 23.707 & 0.2941 & 0.3241 & 6.03 & 1.411 \\
0.75
& 0.950 & 0.918 & 23.925 & \underline{0.2961} & 0.3364 & 6.13 & 1.483 \\
1.0
& 0.958 & 0.926 & 23.988 & \textbf{0.2963} & 0.3426 & 6.20 & 1.504 \\
\rowcolor{lightbluecell}
\textbf{1.25}
& \textbf{0.968} & \underline{0.936} & \textbf{24.034} & 0.2959 & \underline{0.3487} & \textbf{6.29} & \textbf{1.525} \\
1.5
& \underline{0.961} & \textbf{0.938} & \underline{24.019} & 0.2945 & \textbf{0.3505} & \underline{6.26} & \underline{1.523} \\
\bottomrule
\end{tabular}
}
\caption{Ablation results on $\lambda$. In-domain: \textbf{G}en\textbf{E}val, \textbf{OCR}, \textbf{Pick}Score, \textbf{Clip}Score, and \textbf{HPS}v2.1; OOD: \textbf{Aes}thetic and \textbf{I}mage\textbf{R}eward. \textbf{Bold}: best; \underline{Underline}: second best.}
\label{tab:lambda}
\end{table}

\paragraph{Degraded reference.}

\begin{table}[!h]
\centering
{
\small
\renewcommand{\arraystretch}{0.95}

\begin{tabular}{@{}lcccccccc@{}}
\toprule
\multirow{2}{*}[-0.3ex]{\textbf{Method}}
& \multicolumn{5}{c}{\textbf{In-Domain}}
& \multicolumn{2}{c}{\textbf{OOD}}
& \multirow{2}{*}[-0.3ex]{\textbf{Avg.}} \\
\cmidrule(lr){2-6} \cmidrule(lr){7-8}
& \textbf{GE}
& \textbf{OCR}
& \textbf{Pick}
& \textbf{Clip}
& \textbf{HPS}
& \textbf{Aes}
& \textbf{IR}
& \\
\midrule

w/o DeRef.
& \underline{0.967} & 0.928 & 24.03 & \underline{0.2956} & 0.3487 & 6.24 & 1.517 & 0.558 \\
\midrule

\multicolumn{9}{@{}l}{\textit{Reference Velocity Quantization}} \\
4-bit
& 0.962 & 0.923 & 24.03 & 0.2950 & 0.3482 & 6.28 & \underline{1.524} & 0.566 \\
\rowcolor{lightbluecell}
\textbf{8-bit}
& \textbf{0.968} & \textbf{0.936} & \underline{24.04} & \textbf{0.2959} & 0.3487 & \underline{6.29} & \textbf{1.525} & \textbf{0.926} \\
\midrule

\multicolumn{9}{@{}l}{\textit{Reference Weight Quantization}} \\
4-bit
& 0.959 & 0.920 & 23.96 & 0.2929 & 0.3489 & \textbf{6.30} & 1.523 & 0.372 \\
8-bit
& 0.961 & \underline{0.931} & \textbf{24.05} & 0.2949 & \underline{0.3490} & 6.28 & \textbf{1.525} & \underline{0.739} \\
\midrule

\multicolumn{9}{@{}l}{\textit{Reference Velocity Gaussian Perturbation}} \\
$\sigma=0.01$
& 0.959 & 0.923 & 24.01 & 0.2954 & 0.3477 & 6.27 & 1.520 & 0.358 \\
$\sigma=0.1$
& 0.960 & 0.925 & \underline{24.04} & 0.2948 & \textbf{0.3491} & 6.28 & 1.521 & 0.594 \\
% \midrule

% \multicolumn{9}{@{}l}{\textit{Time-Embedding Gaussian Perturbation}} \\
% $\sigma=0.01$
% & .964 & .932 & .9238 & .2953 & .3478 & 6.28 & 1.515 & .521 \\
% $\sigma=0.1$
% & .958 & .928 & .9236 & .2948 & .3489 & 6.28 & 1.524 & .593 \\
\bottomrule
\end{tabular}
}
\caption{Ablation results on degraded reference. Rewards: \textbf{G}en\textbf{E}val, \textbf{OCR}, \textbf{Pick}Score, \textbf{Clip}Score, \textbf{HPS}v2.1, \textbf{Aes}thetic and \textbf{I}mage\textbf{R}eward. \textbf{Bold}: best; \underline{Underline}: second best.}

\label{tab:reference}
\end{table}

Table~\ref{tab:reference} compares different degradation strategies. Moderate degradation performs best, with 8-bit velocity quantization achieving the highest average score. This suggests that a mildly degraded reference can enlarge the teacher-reference contrast and provide a more informative extrapolation direction. However, overly weak degradation gives limited contrast, whereas stronger degradation disrupts the reference structure and yields inconsistent gains by making the extrapolation direction less reliable. Appendix~\ref{app:degraded-reference-construction} reports the metrics of each degraded reference.

\paragraph{Noise level.}
As shown in Figure~\ref{fig:exp_noise}, reducing noise level consistently improves performance. Deterministic ODE sampling performs best, followed by noise levels of $0.3$, $0.5$ and $0.7$. This suggests that the closed-form target does not require stochastic exploration to provide an effective learning signal. At the same noise level of $0.3$, the policy-gradient baseline performs the worst, supporting our analysis that direct policy optimization introduces additional variance and lead to inferior results.

\begin{figure*}[h]
\centering
\includegraphics[width=1.0\textwidth]{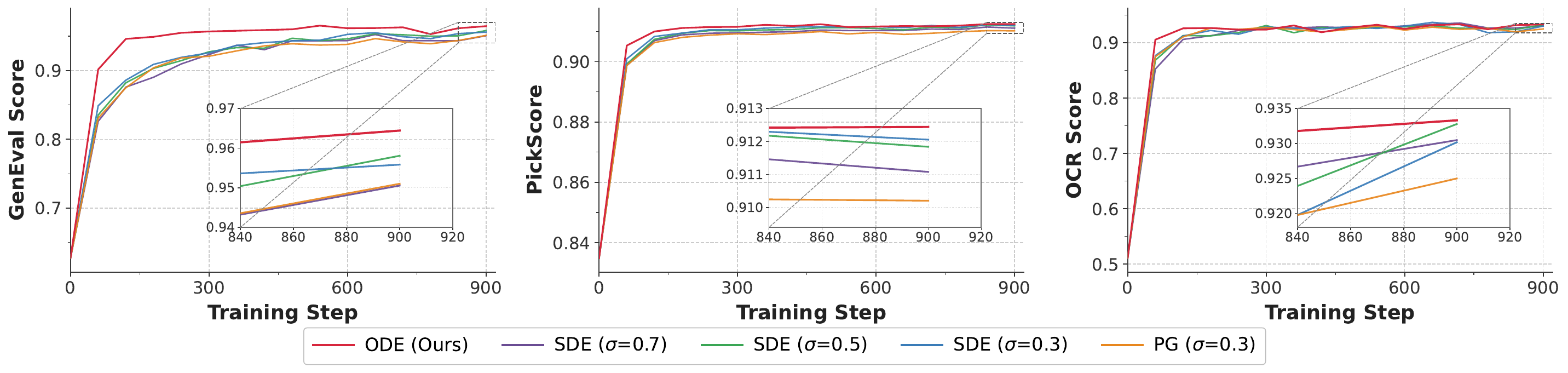}
\caption{Ablation results on the noise level. We compare deterministic ODE sampling with SDE sampling at noise levels $0.3$, $0.5$, and $0.7$, as well as a policy-gradient variant at noise level $0.3$, on GenEval, PickScore, and OCR.}
\label{fig:exp_noise}
\end{figure*}

\section{Conclusion}

We introduced \texttt{\textbf{DreOPD}}, a degraded-reference extrapolative on-policy distillation method for flow-matching models. \texttt{\textbf{DreOPD}} translates trajectory-level implicit reward extrapolation into a closed-form velocity target, extending OPD from teacher imitation to teacher-reference extrapolation while retaining regression-based training. We further showed that a mildly degraded reference can strengthen a reward-aligned contrast without disrupting generative structure. Across multiple settings, \texttt{\textbf{DreOPD}} achieves the best average performance over the baselines, while surpassing teachers on most metrics. These results establish reward extrapolation as an effective framework for consolidating specialized flow models while improving beyond teacher imitation.

\bibliography{ref}
\bibliographystyle{iclr2027_conference}

\clearpage

\appendix
\section{Teacher Surpassing Visualization}

To further visualize the multi-task capability integration, we provide a radar plot in Figure~\ref{fig:app-teacher}. Our \textbf{\texttt{DreOPD}} outperforms the corresponding teacher on six out of seven metrics, with only a slight decrease on ClipScore, and achieves the highest average score overall. This indicates that our student does not merely average teacher behaviors, but successfully consolidates and extrapolates their strengths across tasks.

\begin{figure}[h]
    \centering
    \includegraphics[width=0.5\textwidth]{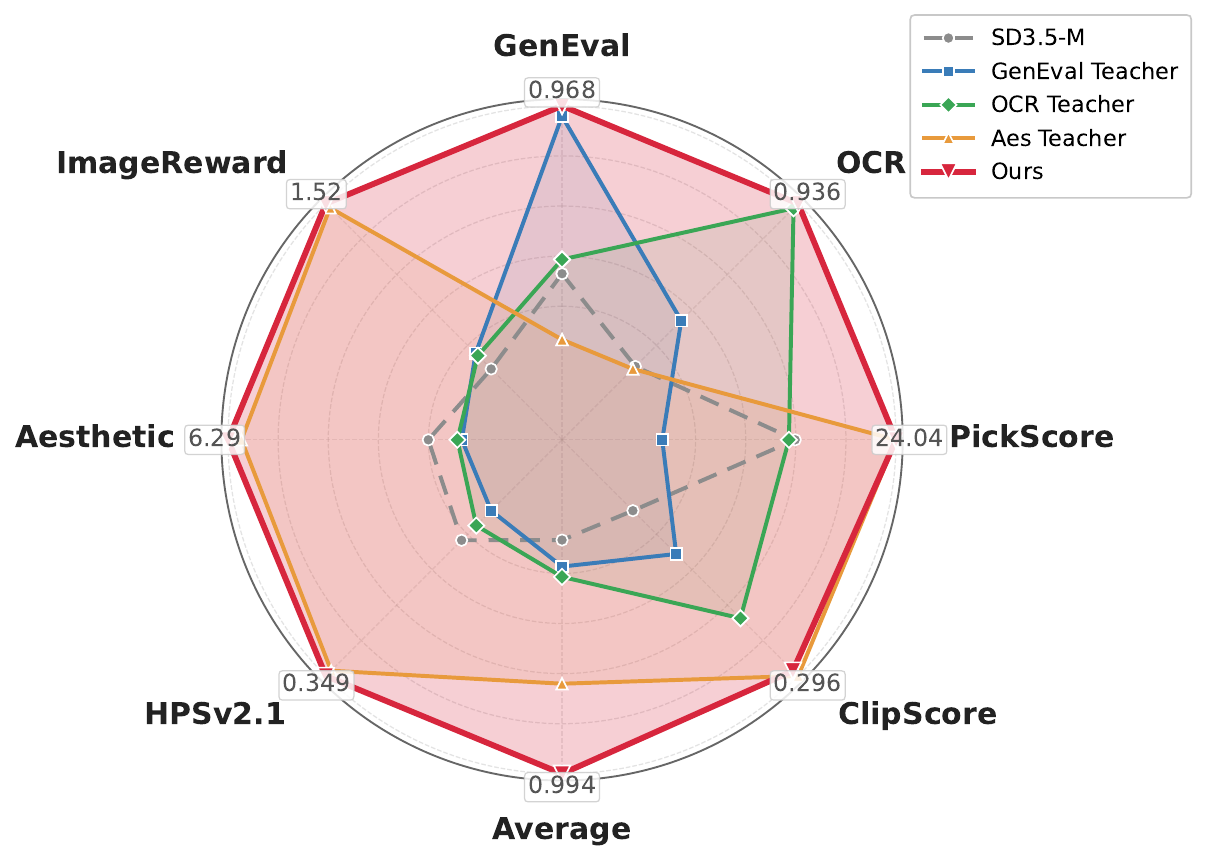}
    \caption{Visualization of multi-task \textbf{\texttt{DreOPD}} performance. We compare the student trained by \textbf{\texttt{DreOPD}} with the corresponding task-specific teachers across metrics. }
    \label{fig:app-teacher}
\end{figure}

\section{Derivations and Proofs}
\label{app:proofs}

\subsection{Derivation of the Closed-Form Velocity Target}
\label{app:velocity-target-derivation}
This section provides the full derivation of the closed-form velocity target in Eq.~(\ref{eq:extrapolated-velocity-target}).

\begin{equation*}
\small
v_\lambda^\star=
v_T
+
(\lambda-1)
\left(
v_T-v_{\mathrm{ref}}
\right).
\end{equation*}

We first decompose the trajectory-level objective into conditional transition objectives and then solve each conditional objective under the shared-covariance Gaussian transition model. Finally, we clarify the relationship between the resulting pointwise optimizer and the on-policy regression objective used for training.

\subsubsection{Trajectory-level Objective Decomposition}
Fix a conditioning input $c$. Following the reverse-time sampling convention, let
\[
\small
1=t_0>t_1>\cdots>t_N=0,
\qquad
\Delta t_j=t_{j+1}-t_j<0.
\]
A trajectory is denoted by
$\tau=(x_{t_0},x_{t_1},\ldots,x_{t_N})$.
The student, teacher, and reference trajectory distributions factorize as
\begin{align}
\small
\Pi_a(\tau\mid c)
=
p(x_{t_0})
\prod_{j=0}^{N-1}
\pi_a^{(j)}
\left(
x_{t_{j+1}}\mid x_{t_j},c
\right),
\label{eq:app-trajectory-factorization}
\end{align}
where $a\in\{\theta,T,\mathrm{ref}\}$ and all three models share the same initial noise distribution $p(x_{t_0})$.

The trajectory-level extrapolation objective is
\begin{align}
\small
\mathcal{J}(\theta;c)
=
\mathbb{E}_{\tau\sim\Pi_\theta(\cdot\mid c)}
\left[
\lambda
\log
\frac{\Pi_T(\tau\mid c)}
     {\Pi_{\mathrm{ref}}(\tau\mid c)}
-
\log
\frac{\Pi_\theta(\tau\mid c)}
     {\Pi_{\mathrm{ref}}(\tau\mid c)}
\right].
\label{eq:app-trajectory-objective}
\end{align}

Because the initial distribution $p(x_{t_0})$ is shared, it cancels from each
trajectory likelihood ratio. In particular,

\begin{equation}
\small
\begin{gathered}
\log
\frac{\Pi_T(\tau\mid c)}
     {\Pi_{\mathrm{ref}}(\tau\mid c)}
=
\log
\frac{
p(x_{t_0})
\prod_{j=0}^{N-1}
\pi_T^{(j)}
\left(
x_{t_{j+1}}\mid x_{t_j},c
\right)
}{
p(x_{t_0})
\prod_{j=0}^{N-1}
\pi_{\mathrm{ref}}^{(j)}
\left(
x_{t_{j+1}}\mid x_{t_j},c
\right)
}=
\sum_{j=0}^{N-1}
\log
\frac{
\pi_T^{(j)}
\left(
x_{t_{j+1}}\mid x_{t_j},c
\right)
}{
\pi_{\mathrm{ref}}^{(j)}
\left(
x_{t_{j+1}}\mid x_{t_j},c
\right)
},
\\
\log
\frac{\Pi_\theta(\tau\mid c)}
     {\Pi_{\mathrm{ref}}(\tau\mid c)}=
\sum_{j=0}^{N-1}
\log
\frac{
\pi_\theta^{(j)}
\left(
x_{t_{j+1}}\mid x_{t_j},c
\right)
}{
\pi_{\mathrm{ref}}^{(j)}
\left(
x_{t_{j+1}}\mid x_{t_j},c
\right)
}.
\end{gathered}
\label{eq:app-teacher-reference-ratio}
\end{equation}
Let $d_{\theta,j}(x_{t_j}\mid c)$ denote the marginal distribution of $x_{t_j}$ induced by the student trajectory distribution $\Pi_\theta(\cdot\mid c)$. Substituting Eq.~(\ref{eq:app-teacher-reference-ratio}) into Eq.~(\ref{eq:app-trajectory-objective}) and applying the tower property gives

\begin{equation}
\small
\begin{aligned}
\mathcal{J}(\theta)
=
\sum_{j=0}^{N-1}
\mathbb{E}_{x_{t_j}\sim d_{\theta,j}(\cdot\mid c)}
\Bigg[
\lambda\,
\mathbb{E}_{
x_{t_{j+1}}
\sim
\pi_\theta^{(j)}
}
\left[
\log
\frac{
\pi_T^{(j)}
}{
\pi_{\mathrm{ref}}^{(j)}
}
\right]
-
\mathrm{KL}
\left(
\pi_\theta^{(j)}
(\cdot\mid x_{t_j},c)
\,\middle\|\,
\pi_{\mathrm{ref}}^{(j)}
(\cdot\mid x_{t_j},c)
\right)
\Bigg].
\end{aligned}
\label{eq:app-trajectory-decomposition}
\end{equation}

Equation~(\ref{eq:app-trajectory-decomposition}) is an exact decomposition of the trajectory-level objective.

\subsubsection{Conditional Gaussian Transition Objective}

At a fixed state $(x_{t_j},t_j,c)$, suppose that the student, teacher, and reference transitions are Gaussian with a shared covariance:
\begin{equation}
\small
\begin{gathered}
\pi_a^{(j)}
\left(
x_{t_{j+1}}\mid x_{t_j},c
\right)
=
\mathcal{N}
\left(
x_{t_{j+1}};
\mu_a^{(j)},
\Sigma_{t_j}
\right),
\\
\mu_a^{(j)}
=
x_{t_j}
+
v_a(x_{t_j},t_j,c)\Delta t_j,
\Sigma_{t_j}=
\sigma_{t_j}^2|\Delta t_j|I
=
\sigma_{t_j}^2(-\Delta t_j)I .
\end{gathered}
\label{eq:app-shared-gaussian-transition}
\end{equation}

The final equality follows from $\Delta t_j<0$. For Gaussian distributions with the same covariance matrix,
\begin{equation}
\small
\mathrm{KL}
\left(
\mathcal{N}(\mu_a,\Sigma)
\,\middle\|\,
\mathcal{N}(\mu_b,\Sigma)
\right)
=
\frac{1}{2}
(\mu_a-\mu_b)^\top
\Sigma^{-1}
(\mu_a-\mu_b).
\label{eq:app-shared-gaussian-kl}
\end{equation}

Moreover, for
$X\sim\mathcal{N}(\mu_\theta,\Sigma)$,
\begin{equation}
\small
\begin{aligned}
&\mathbb{E}
\left[
\log
\frac{
\mathcal{N}(X;\mu_T,\Sigma)
}{
\mathcal{N}(X;\mu_{\mathrm{ref}},\Sigma)
}
\right]
\\
&=
-\frac{1}{2}
\mathbb{E}
\left[
(X-\mu_T)^\top
\Sigma^{-1}
(X-\mu_T)
\right]+\frac{1}{2}
\mathbb{E}
\left[
(X-\mu_{\mathrm{ref}})^\top
\Sigma^{-1}
(X-\mu_{\mathrm{ref}})
\right]
\\
&=
-\frac{1}{2}
\Big[
\operatorname{tr}(\Sigma^{-1}\Sigma)
+
(\mu_\theta-\mu_T)^\top
\Sigma^{-1}
(\mu_\theta-\mu_T)
\Big]+
\frac{1}{2}
\Big[
\operatorname{tr}(\Sigma^{-1}\Sigma)
+
(\mu_\theta-\mu_{\mathrm{ref}})^\top
\Sigma^{-1}
(\mu_\theta-\mu_{\mathrm{ref}})
\Big]
\\
&=
\frac{1}{2}
\left\|
\mu_\theta-\mu_{\mathrm{ref}}
\right\|_{\Sigma^{-1}}^2
-
\frac{1}{2}
\left\|
\mu_\theta-\mu_T
\right\|_{\Sigma^{-1}}^2,
\end{aligned}
\label{eq:app-gaussian-log-ratio}
\end{equation}
where $\|z\|_{\Sigma^{-1}}^2=z^\top\Sigma^{-1}z$. The trace terms cancel because all three transitions share the same covariance.

Applying Eq.~(\ref{eq:app-shared-gaussian-kl}) and Eq.~(\ref{eq:app-gaussian-log-ratio}) at step $j$, the term of expectation in Eq.~(\ref{eq:app-trajectory-decomposition}) becomes
\begin{equation}
\small
\begin{aligned}
\mathcal{J}_j^{\mathrm{cond}}(\mu_\theta^{(j)})
&=
\lambda\,
\mathbb{E}_{\pi_\theta^{(j)}}
\left[
\log
\frac{\pi_T^{(j)}}{\pi_{\mathrm{ref}}^{(j)}}
\right]
-
\mathrm{KL}
\left(
\pi_\theta^{(j)}
\,\middle\|\,
\pi_{\mathrm{ref}}^{(j)}
\right)
\\
&=
\underbrace{
\frac{\lambda}{2}
\left(
\left\|
\mu_\theta^{(j)}
-
\mu_{\mathrm{ref}}^{(j)}
\right\|_{\Sigma_{t_j}^{-1}}^2
-
\left\|
\mu_\theta^{(j)}
-
\mu_T^{(j)}
\right\|_{\Sigma_{t_j}^{-1}}^2
\right)}_{\lambda\,
\mathbb{E}
\left[
\log{\pi_T^{(j)}}/{\pi_{\mathrm{ref}}^{(j)}}
\right]}-
\underbrace{
\frac{1}{2}
\left\|
\mu_\theta^{(j)}
-
\mu_{\mathrm{ref}}^{(j)}
\right\|_{\Sigma_{t_j}^{-1}}^2}_{\mathrm{KL}
\left(
\pi_\theta^{(j)}
\,\middle\|\,
\pi_{\mathrm{ref}}^{(j)}
\right)}
\\
&=
-\frac{1}{2}
\left[
\lambda
\left\|
\mu_\theta^{(j)}
-
\mu_T^{(j)}
\right\|_{\Sigma_{t_j}^{-1}}^2
\right.\left.
-
(\lambda-1)
\left\|
\mu_\theta^{(j)}
-
\mu_{\mathrm{ref}}^{(j)}
\right\|_{\Sigma_{t_j}^{-1}}^2
\right].
\end{aligned}
\label{eq:app-conditional-objective}
\end{equation}

Therefore, under $\Sigma_{t_j}^{-1}=[\sigma_{t_j}^2(-\Delta t_j)]^{-1}I$, maximizing the conditional objective is equivalent to minimizing
\begin{equation}
\small
\begin{aligned}
\ell_j(\mu_\theta)
&=
\frac{1}{2}
\left[
\lambda
\left\|
\mu_\theta^{(j)}
-
\mu_T^{(j)}
\right\|_{\Sigma_{t_j}^{-1}}^2
-
(\lambda-1)
\left\|
\mu_\theta^{(j)}
-
\mu_{\mathrm{ref}}^{(j)}
\right\|_{\Sigma_{t_j}^{-1}}^2
\right]
\\
&=
\frac{
\lambda
\left\|
\mu_\theta^{(j)}-\mu_T^{(j)}
\right\|^2
-
(\lambda-1)
\left\|
\mu_\theta^{(j)}-\mu_{\mathrm{ref}}^{(j)}
\right\|^2
}{
2\sigma_{t_j}^2(-\Delta t_j)
}.
\end{aligned}
\label{eq:app-conditional-mean-loss}
\end{equation}

\subsubsection{Closed-Form Optimizer in Velocity Space}

From Eq.~(\ref{eq:app-shared-gaussian-transition}), the difference between any two transition means satisfies
\begin{equation}
\small
\mu_a^{(j)}-\mu_b^{(j)}
=
\Delta t_j
\left[
v_a(x_{t_j},t_j,c)
-
v_b(x_{t_j},t_j,c)
\right].
\label{eq:app-mean-velocity-difference}
\end{equation}

Substituting Eq.~(\ref{eq:app-mean-velocity-difference}) into Eq.~(\ref{eq:app-conditional-mean-loss}) gives
\begin{equation}
\small
\begin{aligned}
\ell_j(v_\theta)
&=
\frac{(\Delta t_j)^2}
     {2\sigma_{t_j}^2(-\Delta t_j)}
\left[
\lambda
\left\|
v_\theta-v_T
\right\|^2
-
(\lambda-1)
\left\|
v_\theta-v_{\mathrm{ref}}
\right\|^2
\right]
\\
&=
\kappa_{t_j}
\left[
\lambda
\left\|
v_\theta-v_T
\right\|^2
-
(\lambda-1)
\left\|
v_\theta-v_{\mathrm{ref}}
\right\|^2
\right],
\end{aligned}
\label{eq:app-conditional-velocity-loss}
\end{equation}

where $\kappa_{t_j}={-\Delta t_j}/{2\sigma_{t_j}^2}>0$. Although Eq.~(\ref{eq:app-conditional-velocity-loss}) contains a negative quadratic term, the overall objective is strictly convex in $v_\theta$. To see this, expand the two squared distances:
\begin{equation}
\small
\begin{aligned}
\frac{\ell_j(v_\theta)}{\kappa_{t_j}}&=
\lambda
\left(
\|v_\theta\|^2
-
2v_\theta^\top v_T
+
\|v_T\|^2
\right)-
(\lambda-1)
\left(
\|v_\theta\|^2
-
2v_\theta^\top v_{\mathrm{ref}}
+
\|v_{\mathrm{ref}}\|^2
\right)
\\
&=
\|v_\theta\|^2
-
2v_\theta^\top
\left[
\lambda v_T-(\lambda-1)v_{\mathrm{ref}}
\right]+
\lambda\|v_T\|^2
-
(\lambda-1)\|v_{\mathrm{ref}}\|^2.
\label{eq:app-velocity-loss-expanded}
\end{aligned}
\end{equation}

The coefficient of $\|v_\theta\|^2$ is $\lambda-(\lambda-1)=1$. Consequently, the Hessian is $\nabla_{v_\theta}^2\ell_j(v_\theta)=2\kappa_{t_j}I\succ0$ and the conditional objective has a unique finite minimizer.

Define $v_\lambda^\star=
\lambda v_T-(\lambda-1)v_{\mathrm{ref}}=
v_T+(\lambda-1)(v_T-v_{\mathrm{ref}})$. Completing the square in Eq.~(\ref{eq:app-velocity-loss-expanded}) yields

\begin{equation}
\small
\begin{aligned}
\frac{\ell_j(v_\theta)}{\kappa_{t_j}}
&=
\left\|
v_\theta-v_\lambda^\star
\right\|^2
+
\lambda\|v_T\|^2
-
(\lambda-1)\|v_{\mathrm{ref}}\|^2
-
\left\|v_\lambda^\star\right\|^2
\\&=
\left\|
v_\theta-v_\lambda^\star
\right\|^2
-
\lambda(\lambda-1)
\left\|
v_T-v_{\mathrm{ref}}
\right\|^2.
\end{aligned}
\label{eq:app-completing-square}
\end{equation}

The second term $\lambda(\lambda-1)\left\|v_T-v_{\mathrm{ref}}\right\|^2$ is independent of $v_\theta$. Hence, $\arg\min_{v_\theta}\ell_j(v_\theta)=v_\lambda^\star$.

\subsection{From SDE to ODE}
\label{app:sde-to-ode}

Our closed-form target is derived from stochastic transitions with a shared Gaussian covariance. In practice, we follow DiffusionOPD~\citep{li2026diffusionopd} and use deterministic ODE rollouts, which perform best in our ablation. The SDE conditional objective can be written as
\begin{equation}
\small
\begin{aligned}
&\ell_j^{\mathrm{SDE}}(v_\theta)
=
\kappa_{t_j}
\left[
\lambda\|v_\theta-v_T\|^2
-
(\lambda-1)\|v_\theta-v_{\mathrm{ref}}\|^2
\right]=
\kappa_{t_j}
\left\|
v_\theta-v_\lambda^\star
\right\|^2
+
C_j,
\end{aligned}
\label{eq:app-sde-target}
\end{equation}
where $v_\lambda^\star=v_T+(\lambda-1)(v_T-v_{\mathrm{ref}})$, $\kappa_{t_j}=-\Delta t_j/(2\sigma_{t_j}^2)>0$ and $C_j$ is independent of $v_\theta$ at a fixed visited state. Therefore, the pointwise optimizer $v_\lambda^\star$ is independent of the shared transition noise.

Under the deterministic ODE update, the student, teacher, and reference induce the transition means
\begin{equation}
\small
\begin{aligned}
\mu_a^{(j)}
&=
x_{t_j}
+
v_a(x_{t_j},t_j,c)\Delta t_j,
\quad
a\in\{\theta,T,\mathrm{ref}\}.
\end{aligned}
\label{eq:app-ode-transition-means}
\end{equation}

The extrapolated transition target is therefore
\begin{equation}
\small
\begin{aligned}
\mu_\lambda^\star
=
\mu_T^{(j)}
+
(\lambda-1)
\left(
\mu_T^{(j)}-\mu_{\mathrm{ref}}^{(j)}
\right)
=
x_{t_j}
+
v_\lambda^\star(x_{t_j},t_j,c)\Delta t_j.
\end{aligned}
\label{eq:app-ode-extrapolated-mean}
\end{equation}

For the deterministic transition-matching objective, we directly regress the student transition mean toward $\mu_\lambda^\star$. Since
$\mu_\theta^{(j)}-\mu_\lambda^\star
=
\Delta t_j(v_\theta-v_\lambda^\star)$, the ODE objective becomes
\begin{equation}
\small
\begin{aligned}
\mathcal{L}^{\mathrm{ODE}}(\theta)
&=
\mathbb{E}_{c}
\left[
\sum_{j=0}^{N-1}
\frac{1}{2}
\left\|
\mu_\theta^{(j)}
-
\mathrm{sg}\!\left(\mu_\lambda^\star\right)
\right\|^2
\right]
\\
&=
\mathbb{E}_{c}
\left[
\sum_{j=0}^{N-1}
\frac{(\Delta t_j)^2}{2}
\left\|
v_\theta(x_{t_j},t_j,c)
-
\mathrm{sg}\!\left(
v_\lambda^\star(x_{t_j},t_j,c)
\right)
\right\|^2
\right].
\end{aligned}
\label{eq:app-ode-regression}
\end{equation}

The states $\{x_{t_j}\}_{j=0}^{N}$ are collected from the current student ODE rollout and treated as fixed during each regression update. Thus, the ODE formulation retains the closed-form extrapolated target while replacing stochastic transition matching with direct deterministic mean matching.

\subsection{Proof of Lemma~\ref{lem:teacher-reward-tilt}}
\label{app:proof-reward-tilt}

\begin{proof}
Define the reward-tilted distribution
\begin{equation}
\small
\begin{aligned}
q(x)
=
\frac{1}{Z_T}
p_{\mathrm{ref}}(x)
\exp\left(\frac{r(x)}{\beta}\right), \quad
Z_T=
\int
p_{\mathrm{ref}}(x)
\exp\left(\frac{r(x)}{\beta}\right)
dx.
\end{aligned}
\label{eq:app-reward-tilted-distribution}
\end{equation}

Because $r$ is bounded and $p_{\mathrm{ref}}$ is a probability distribution, $0<Z_T<\infty$, and hence $q$ is a well-defined probability distribution. Moreover, $q$ and $p_{\mathrm{ref}}$ have the same support.

For any probability distribution $p$ satisfying $p\ll p_{\mathrm{ref}}$, the KL-regularized reward objective can be rewritten as
\begin{equation}
\small
\begin{aligned}
\mathbb{E}_{x\sim p}[r(x)]
-
\beta\,
\mathrm{KL}
\left(
p\,\middle\|\,p_{\mathrm{ref}}
\right)
&=
\int
p(x)
\left[
r(x)
-
\beta
\log
\frac{p(x)}{p_{\mathrm{ref}}(x)}
\right]
dx
\\
&=
-\beta
\int
p(x)
\log
\frac{
p(x)
}{
p_{\mathrm{ref}}(x)
\exp\left(r(x)/\beta\right)
}
dx
\\
&=
-\beta
\int
p(x)
\log
\frac{p(x)}{Z_T q(x)}
dx
\\
&=
-\beta
\int
p(x)
\log
\frac{p(x)}{q(x)}
dx
+
\beta\log Z_T
\\
&=
\beta\log Z_T
-
\beta\,
\mathrm{KL}
\left(
p\,\middle\|\,q
\right).
\end{aligned}
\label{eq:app-reward-objective-kl-decomposition}
\end{equation}

Since $\mathrm{KL}(p\|q)\geq0$, with equality if and only if $p=q$ almost everywhere, the objective is uniquely maximized by $p_T=q$. Therefore,
\begin{equation}
\small
p_T(x)
=
\frac{1}{Z_T}
p_{\mathrm{ref}}(x)
\exp\left(\frac{r(x)}{\beta}\right),
\end{equation}
which proves the result.
\end{proof}

\subsection{Proof of Proposition~\ref{prop:reward-monotonicity}}
\label{app:proof-reward-monotonicity}

\subsubsection{Optimal Extrapolated Distribution}
\begin{proof}
For compactness, define the teacher-reference log-density ratio $g(x)=\log\frac{p_T(x)}{p_{\mathrm{ref}}(x)}$ on the support of $p_{\mathrm{ref}}$. The distributional objective can then be written as
\begin{equation}
\small
\begin{aligned}
\mathcal{J}_\lambda(p)
&=
\lambda
\int p(x)g(x)\,dx
-
\int
p(x)
\log
\frac{p(x)}{p_{\mathrm{ref}}(x)}
dx
\\
&=
-\int
p(x)
\log
\frac{p(x)}
{
p_{\mathrm{ref}}(x)
\exp\left(\lambda g(x)\right)
}
dx.
\end{aligned}
\label{eq:app-extrapolation-objective-density-ratio}
\end{equation}

Define
\begin{equation}
\small
\begin{aligned}
Z_\lambda
&=
\int
p_{\mathrm{ref}}(x)
\exp\left(\lambda g(x)\right)
dx=
\int
p_{\mathrm{ref}}(x)
\left(
\frac{p_T(x)}{p_{\mathrm{ref}}(x)}
\right)^\lambda
dx,
\\
p_\lambda(x)
&=
\frac{1}{Z_\lambda}
p_{\mathrm{ref}}(x)
\exp\left(\lambda g(x)\right)=
\frac{1}{Z_\lambda}
p_T(x)^\lambda
p_{\mathrm{ref}}(x)^{1-\lambda}.
\end{aligned}
\label{eq:app-extrapolated-distribution}
\end{equation}

By assumption, $Z_\lambda$ is finite, so $p_\lambda$ is well defined.
Substituting $p_{\mathrm{ref}}(x)\exp(\lambda g(x))=Z_\lambda p_\lambda(x)$
into Eq.~(\ref{eq:app-extrapolation-objective-density-ratio}) yields
\begin{equation}
\small
\begin{aligned}
\mathcal{J}_\lambda(p)
&=
-\int
p(x)
\log
\frac{p(x)}
{
Z_\lambda p_\lambda(x)
}
dx
=
-\int
p(x)
\log
\frac{p(x)}{p_\lambda(x)}
dx
+
\log Z_\lambda
=
\log Z_\lambda
-
\mathrm{KL}
\left(
p\,\middle\|\,p_\lambda
\right).
\end{aligned}
\label{eq:app-extrapolation-objective-kl-decomposition}
\end{equation}
Since the KL divergence is nonnegative and vanishes if and only if
$p=p_\lambda$ almost everywhere, $p_\lambda$ is the unique optimizer
of $\mathcal{J}_\lambda(p)$.
\end{proof}

\subsubsection{Monotonicity of the Expected Reward}

\begin{proof}
Under the reward-tilt assumption in Lemma~1, the teacher-reference log-density ratio satisfies
\begin{equation}
\small
\begin{aligned}
g(x)=
\log
\frac{p_T(x)}{p_{\mathrm{ref}}(x)}=
\frac{r(x)}{\beta}-\log Z_T.
\end{aligned}
\label{eq:app-log-ratio-as-reward}
\end{equation}

Substituting Eq.~(\ref{eq:app-log-ratio-as-reward}) into
Eq.~(\ref{eq:app-extrapolated-distribution}) gives
\begin{equation}
\small
\begin{aligned}
p_\lambda(x)
&=
\frac{
p_{\mathrm{ref}}(x)
\exp\left(\lambda g(x)\right)
}{
\int
p_{\mathrm{ref}}(x')
\exp\left(\lambda g(x')\right)
dx'
}
\\
&=
\frac{
p_{\mathrm{ref}}(x)
\exp\left(\lambda r(x)/\beta\right)
\exp\left(-\lambda\log Z_T\right)
}{
\int
p_{\mathrm{ref}}(x')
\exp\left(\lambda r(x')/\beta\right)
\exp\left(-\lambda\log Z_T\right)
dx'
}
\\
&=
\frac{
p_{\mathrm{ref}}(x)
\exp\left(\lambda r(x)/\beta\right)
}{
\widetilde{Z}_\lambda
},
\\
\widetilde{Z}_\lambda
&=
\int
p_{\mathrm{ref}}(x)
\exp\left(\lambda r(x)/\beta\right)
dx.
\end{aligned}
\label{eq:app-extrapolated-reward-tilt}
\end{equation}

Because $r$ is bounded, differentiation under the integral sign is valid. Differentiating the log-density of $p_\lambda$ with respect to $\lambda$ yields
\begin{equation}
\small
\begin{aligned}
\frac{\partial}{\partial\lambda}
\log p_\lambda(x)
&=
\frac{r(x)}{\beta}
-
\frac{\partial}{\partial\lambda}
\log\widetilde{Z}_\lambda
\\
&=
\frac{r(x)}{\beta}
-
\frac{1}{\widetilde{Z}_\lambda}
\int
p_{\mathrm{ref}}(x')
\exp\left(\frac{\lambda r(x')}{\beta}\right)
\frac{r(x')}{\beta}
dx'
\\
&=
\frac{1}{\beta}
\left(
r(x)
-
\mathbb{E}_{x'\sim p_\lambda}[r(x')]
\right)
\\
&=
\frac{1}{\beta}
\left(
r(x)-J(\lambda)
\right).
\end{aligned}
\label{eq:app-log-density-derivative}
\end{equation}

Therefore,
\begin{equation}
\small
\begin{aligned}
\frac{dJ(\lambda)}{d\lambda}
&=
\frac{d}{d\lambda}
\int
p_\lambda(x)r(x)\,dx
\\
&=
\int
r(x)
\frac{\partial p_\lambda(x)}{\partial\lambda}
dx
\\
&=
\int
r(x)p_\lambda(x)
\frac{\partial\log p_\lambda(x)}{\partial\lambda}
dx
\\
&=
\frac{1}{\beta}
\int
p_\lambda(x)r(x)
\left(
r(x)-J(\lambda)
\right)
dx
\\
&=
\frac{1}{\beta}
\left(
\mathbb{E}_{x\sim p_\lambda}[r(x)^2]
-
\mathbb{E}_{x\sim p_\lambda}[r(x)]^2
\right)
\\
&=
\frac{1}{\beta}
\mathrm{Var}_{x\sim p_\lambda}[r(x)]
\geq0.
\end{aligned}
\label{eq:app-reward-monotonicity-proof}
\end{equation}

The inequality is strict exactly when $\mathrm{Var}_{p_\lambda}[r(x)]>0$, or equivalently, when $r$ is not constant $p_\lambda$-almost surely. This proves the claimed monotonicity.
\end{proof}

\subsection{Proof of Proposition~\ref{thm:reference-amplification}}
\label{app:proof-reference-amplification}

Recall the reward-aligned distribution family
\begin{equation}
\small
\begin{aligned}
p_\rho(x)=
\frac{1}{Z_\rho}
p_0(x)
\exp\left(\frac{\rho r(x)}{\beta}\right),\quad
Z_\rho=
\int
p_0(x)
\exp\left(\frac{\rho r(x)}{\beta}\right)dx.
\end{aligned}
\label{eq:app-reward-alignment-family}
\end{equation}

In particular, the teacher is given by
\begin{equation}
\small
p_T(x)
=
p_1(x)
=
\frac{1}{Z_1}
p_0(x)
\exp\left(\frac{r(x)}{\beta}\right).
\label{eq:app-teacher-reward-alignment-family}
\end{equation}

Because $r$ is bounded, all normalizing constants appearing below are
finite and strictly positive.

\subsubsection{Effective Reward-Tilt Coefficient}
\begin{proof}
Using $p_T=p_1$ and $p_{\mathrm{ref}}=p_\rho$, the unnormalized
extrapolated distribution can be expanded as
\begin{equation}
\small
\begin{aligned}
p_T(x)^\lambda p_\rho(x)^{1-\lambda}
&=
\left[
\frac{1}{Z_1}
p_0(x)
\exp\left(\frac{r(x)}{\beta}\right)
\right]^\lambda
\left[
\frac{1}{Z_\rho}
p_0(x)
\exp\left(\frac{\rho r(x)}{\beta}\right)
\right]^{1-\lambda}
\\
&=
Z_1^{-\lambda}
Z_\rho^{\lambda-1}
p_0(x)^{\lambda+1-\lambda}
\exp\left(
\frac{
\lambda r(x)+(1-\lambda)\rho r(x)
}{\beta}
\right)
\\
&=
Z_1^{-\lambda}
Z_\rho^{\lambda-1}
p_0(x)
\exp\left(
\frac{
[\lambda+(1-\lambda)\rho]r(x)
}{\beta}
\right)
\\
&=
Z_1^{-\lambda}
Z_\rho^{\lambda-1}
p_0(x)
\exp\left(
\frac{
[1+(\lambda-1)(1-\rho)]r(x)
}{\beta}
\right).
\end{aligned}
\label{eq:app-effective-reward-tilt-derivation}
\end{equation}

The factor $Z_1^{-\lambda}Z_\rho^{\lambda-1}$ is independent of $x$ and is therefore absorbed into the normalizing constant. Defining
\begin{equation}
\small
\begin{aligned}
\rho_{\mathrm{eff}}(\lambda,\rho)
=
1+(\lambda-1)(1-\rho), \quad
Z_{\lambda,\rho}=
\int
p_0(x)
\exp\left(
\frac{
\rho_{\mathrm{eff}}(\lambda,\rho)r(x)
}{\beta}
\right)dx,
\end{aligned}
\label{eq:app-effective-alignment-coefficient}
\end{equation}
the normalized extrapolated distribution is
\begin{equation}
\small
p_{\lambda,\rho}(x)
=
\frac{1}{Z_{\lambda,\rho}}
p_0(x)
\exp\left(
\frac{
\rho_{\mathrm{eff}}(\lambda,\rho)r(x)
}{\beta}
\right).
\label{eq:app-effective-alignment-distribution}
\end{equation}

Thus, $p_{\lambda,\rho}$ remains in the same reward-aligned exponential family, with effective alignment coefficient $\rho_{\mathrm{eff}}(\lambda,\rho)$.

Now consider two references satisfying $\rho_d<\rho_n<1$ and fix $\lambda>1$. Since $\lambda-1>0$, we obtain
\begin{equation}
\small
\begin{gathered}
\rho_{\mathrm{eff}}(\lambda,\rho_d)
-
\rho_{\mathrm{eff}}(\lambda,\rho_n)
=
(\lambda-1)
\left[
(1-\rho_d)-(1-\rho_n)
\right]=
(\lambda-1)(\rho_n-\rho_d)>0,
\\
\rho_{\mathrm{eff}}(\lambda,\rho_n)-1
=
(\lambda-1)(1-\rho_n)>0.
\end{gathered}
\label{eq:app-effective-alignment-ordering}
\end{equation}

Consequently,
\begin{equation}
\small
\rho_{\mathrm{eff}}(\lambda,\rho_d)
>
\rho_{\mathrm{eff}}(\lambda,\rho_n)
>
1,
\end{equation}
which proves the first part of the proposition.
\end{proof}

\subsubsection{Local Reward Sensitivity}

\begin{proof}
For a fixed reference parameter $\rho$, define
\begin{equation}
\small
\begin{gathered}
J_\rho(\lambda)=
\mathbb{E}_{x\sim p_{\lambda,\rho}}[r(x)],
\\
p_{\lambda,\rho}(x)
=
\frac{1}{Z_{\lambda,\rho}}
p_0(x)
\exp\left(
\frac{
\rho_{\mathrm{eff}}(\lambda,\rho)r(x)
}{\beta}
\right),\quad
\rho_{\mathrm{eff}}(\lambda,\rho)
=
1+(\lambda-1)(1-\rho).
\end{gathered}
\label{eq:app-local-sensitivity-definitions}
\end{equation}

The effective alignment coefficient satisfies
\begin{equation}
\small
\frac{\partial
\rho_{\mathrm{eff}}(\lambda,\rho)}
{\partial\lambda}
=
1-\rho.
\label{eq:app-effective-coefficient-derivative}
\end{equation}

Because $r$ is bounded, differentiation under the integral sign is valid. Differentiating the log normalizer gives
\begin{equation}
\small
\begin{aligned}
\frac{\partial}{\partial\lambda}
\log Z_{\lambda,\rho}
&=
\frac{1}{Z_{\lambda,\rho}}
\int
p_0(x)
\exp\left(
\frac{
\rho_{\mathrm{eff}}(\lambda,\rho)r(x)
}{\beta}
\right)
\frac{r(x)}{\beta}
\frac{
\partial \rho_{\mathrm{eff}}(\lambda,\rho)
}{
\partial\lambda
}
dx
\\
&=
\frac{1-\rho}{\beta}
\int
p_{\lambda,\rho}(x)r(x)\,dx
\\
&=
\frac{1-\rho}{\beta}
J_\rho(\lambda).
\end{aligned}
\label{eq:app-reference-log-normalizer-derivative}
\end{equation}

It follows that the derivative of the log density is
\begin{equation}
\small
\begin{aligned}
\frac{\partial}{\partial\lambda}
\log p_{\lambda,\rho}(x)
&=
\frac{r(x)}{\beta}
\frac{\partial
\rho_{\mathrm{eff}}(\lambda,\rho)}
{\partial\lambda}
-
\frac{\partial}{\partial\lambda}
\log Z_{\lambda,\rho}
\\
&=
\frac{1-\rho}{\beta}
r(x)
-
\frac{1-\rho}{\beta}
J_\rho(\lambda)
\\
&=
\frac{1-\rho}{\beta}
\left[
r(x)-J_\rho(\lambda)
\right].
\end{aligned}
\label{eq:app-reference-log-density-derivative}
\end{equation}

The reward sensitivity with respect to $\lambda$ is therefore
\begin{equation}
\small
\begin{aligned}
\frac{dJ_\rho(\lambda)}{d\lambda}
&=
\frac{d}{d\lambda}
\int
p_{\lambda,\rho}(x)r(x)\,dx
\\
&=
\int
r(x)
\frac{\partial p_{\lambda,\rho}(x)}
{\partial\lambda}
dx
\\
&=
\int
p_{\lambda,\rho}(x)r(x)
\frac{\partial}{\partial\lambda}
\log p_{\lambda,\rho}(x)
dx
\\
&=
\frac{1-\rho}{\beta}
\int
p_{\lambda,\rho}(x)r(x)
\left[
r(x)-J_\rho(\lambda)
\right]dx
\\
&=
\frac{1-\rho}{\beta}
\left(
\mathbb{E}_{p_{\lambda,\rho}}[r(x)^2]
-
\mathbb{E}_{p_{\lambda,\rho}}[r(x)]^2
\right)
\\
&=
\frac{1-\rho}{\beta}
\mathrm{Var}_{x\sim p_{\lambda,\rho}}[r(x)].
\end{aligned}
\label{eq:app-reference-reward-sensitivity}
\end{equation}

At $\lambda=1$, $p_{1,\rho}(x)=\frac{1}{Z_1} p_0(x)\exp\left(\frac{r(x)}{\beta}\right)=p_T(x)$. The effective alignment coefficient is independent of the reference. Hence, all reference choices recover the same teacher distribution at $\lambda=1$, and Eq.~(\ref{eq:app-reference-reward-sensitivity}) reduces to
\begin{equation}
\small
\left.
\frac{dJ_\rho(\lambda)}{d\lambda}
\right|_{\lambda=1}
=
\frac{1-\rho}{\beta}
\mathrm{Var}_{x\sim p_T}[r(x)].
\label{eq:app-local-reference-sensitivity}
\end{equation}

Because $r$ is nonconstant under $p_T$, $\mathrm{Var}_{p_T}[r(x)]>0$. For $\rho_d<\rho_n<1$, both local derivatives are therefore strictly positive, and their ratio is
\begin{equation}
\small
\begin{aligned}
\left.
\frac{
dJ_{\rho_d}(\lambda)/d\lambda
}{
dJ_{\rho_n}(\lambda)/d\lambda
}
\right|_{\lambda=1}
&=
\frac{
\frac{1-\rho_d}{\beta}
\mathrm{Var}_{p_T}[r(x)]
}{
\frac{1-\rho_n}{\beta}
\mathrm{Var}_{p_T}[r(x)]
}
=
\frac{1-\rho_d}{1-\rho_n}>1,
\end{aligned}
\label{eq:app-reference-amplification-ratio-proof}
\end{equation}
where the final inequality follows from
$\rho_d<\rho_n<1$. This proves the local reward-sensitivity claim.
\end{proof}

\section{Experimental Details}
\label{app:experimental-details}
\subsection{Training Configuration}
\label{app:training-configuration}

All experiments are conducted on a single node with eight NVIDIA A100
GPUs. We follow the experimental configurations of DiffusionOPD~\citep{li2026diffusionopd}. All experiments use \textbf{\texttt{SD3.5-M}} at a resolution of $512\times512$. We fine-tune LoRA adapters with rank $r=32$ and scaling factor $\alpha=64$, while keeping the teacher and reference models frozen. We optimize the LoRA parameters using AdamW with a learning rate of $3\times10^{-4}$, $\beta_1=0.9$, $\beta_2=0.999$, weight decay $10^{-4}$, and $\epsilon=10^{-8}$. Student trajectories are collected using a 10-step first-order ODE sampler, and evaluation uses 40 sampling steps. Unless otherwise specified, the remaining optimization, sampling, and task-specific configurations follow the corresponding baseline settings.

\subsection{Teachers and Baselines}
\label{app:models-teachers}
All methods are initialized from the same \textbf{\texttt{SD3.5-M}}~\citep{esser2024sd} checkpoint. For OPD-based methods, the task-specific teachers remain frozen throughout training and are queried only at states
visited by the current student. Unless otherwise specified, we follow the teacher construction and baseline configurations adopted by DiffusionOPD~\citep{li2026diffusionopd}. We summarize the relevant
details below.

\paragraph{Single-task RL teachers.}

We use three independently trained teachers, each specialized for one of the target capabilities: compositional prompt following, text rendering, and aesthetic quality. These are the same task-specific teachers used by DiffusionOPD~\citep{li2026diffusionopd}.

The \textbf{\texttt{GenEval teacher}} is trained using DiffusionNFT~\citep{zheng2025diffusionnft} to optimize the GenEval compositional reward. GenEval evaluates compositional prompt following through rule-based correctness signals covering object identity, counting, color, spatial relations, and attribute binding. The \textbf{\texttt{OCR teacher}} is trained using GRPO-Guard~\citep{wang2026grpo} to optimize the text-rendering reward. The \textbf{\texttt{Aesthetics teacher}} is also trained using GRPO-Guard and optimizes the equally weighted reward of PickScore, ClipScore and HPSv2.1.

Each teacher is trained only on its corresponding prompt distribution and reward objective. Consequently, the reported \textbf{\texttt{Teacher}} results in the single-teacher experiment combine the in-domain scores of three different specialized models and do not represent a single multi-task teacher.

\paragraph{Multi-task RL baselines.}

We compare against direct multi-task optimization with Flow-GRPO, GRPO-Guard, DiffusionNFT, and CascadeNFT. These baselines receive the same task-specific prompt datasets and reward definitions used to construct the teachers, but optimize a single shared model rather than training separate specialists.

For \textbf{\texttt{Flow-GRPO}}~\citep{liu2026flow}, the reverse-time flow sampler is treated as a stochastic policy with Gaussian transition kernels, and the model is optimized using group-relative policy updates. In the multi-task setting, training alternates among the GenEval, OCR, and aesthetics prompt datasets, applying the reward associated with the sampled task. \textbf{\texttt{GRPO-Guard}}~\citep{wang2026grpo} follows the same alternating multi-task organization while incorporating its guarded optimization mechanism to improve training stability under potentially exploitable reward signals.

\textbf{\texttt{DiffusionNFT}}~\citep{zheng2025diffusionnft} performs direct reward-based fine-tuning through the differentiable generation process. Its multi-task variant alternates among the three task datasets and applies the corresponding reward objective at each update. We use the same task curriculum across jointly trained multi-task baselines.

The original multi-task formulation of DiffusionNFT adopts sequential optimization. We denote this variant as \textbf{\texttt{CascadeNFT}}. The model is first fine-tuned on one task and then successively adapted to the remaining tasks using their respective objectives.

\paragraph{OPD baselines.}

To enable a fair comparison, all OPD methods use the same student initialization, task-specific teachers, prompt datasets, and task-sampling schedule. 

\textbf{\texttt{Flow-OPD}}~\citep{fang2026flowopd} regresses the student velocity toward the selected teacher velocity on student-generated states. It additionally employs Manifold Anchor Regularization to constrain the student toward a high-quality visual manifold.

\textbf{\texttt{DiffusionOPD}}~\citep{li2026diffusionopd} minimizes the stepwise reverse KL divergence from the student transition to the corresponding teacher transition along student-generated trajectories. 

\subsection{Rewards and Metrics}
\label{app:reward-metrics}

We evaluate the models using the following metrics:
\begin{itemize}
    \item \textbf{GenEval}~\citep{ghosh2023geneval} measures compositional prompt following, including object identity, counting,
    color, spatial relations, and attribute binding.

    \item \textbf{OCR} measures the accuracy of text rendered in generated images using the prompt split released with Flow-GRPO~\citep{liu2026flow}. Here, we perform text recognition using \texttt{PaddleOCR=3.7.0} with \texttt{PaddlePaddle-GPU=3.3.1}.

    \item \textbf{PickScore}~\citep{kirstain2023pick} is a preference-based metric trained to assess human preferences for text-to-image generations.

    \item \textbf{ClipScore}~\citep{hessel2021clipscore} measures the semantic correspondence between a generated image and its text prompt using CLIP representations.

    \item \textbf{HPSv2.1}~\citep{wu2023hps} evaluates text-image alignment and perceptual quality using a model trained on human preference data.

    \item \textbf{Aesthetic Score}~\citep{schuhmann2022aesthetics} estimates the perceptual quality of generated images. We use it as an out-of-domain metric because it is not directly included in the training reward.

    \item \textbf{ImageReward}~\citep{xu2023imagereward} evaluates human preference and text-image alignment. It is also used only for out-of-domain evaluation.
\end{itemize}

\paragraph{Average score.}

Because the metrics have different numerical ranges, we compute the average score by independently applying min-max normalization to each metric. Let $s_{m,k}$ denote the score of method $m$ on metric $k$, and let $\mathcal{M}$ denote the set of methods included in the corresponding table. We compute
\begin{equation}
\begin{aligned}
\widetilde{s}_{m,k}
&=
\frac{s_{m,k}-s_k^{\min}}
{s_k^{\max}-s_k^{\min}}, \quad
\operatorname{Avg}(m)
=
\frac{1}{K}
\sum_{k=1}^{K}
\widetilde{s}_{m,k},
\end{aligned}
\label{eq:normalized-average}
\end{equation}
where $K$ is the number of metrics included in the table. All metrics are positively oriented, such that a larger value indicates better performance.

\subsection{Degraded-Reference Construction}
\label{app:degraded-reference-construction}

The reference is the pretrained \textbf{\texttt{SD3.5-M}} backbone without LoRA adapters and with classifier-free guidance disabled. We keep the sampler, noise schedule, and guidance scale fixed across all degradation variants and modify only the reference velocity or backbone weights.

\paragraph{Velocity Quantization.}
We apply per-channel quantization directly to the reference velocity output using either 4-bit or 8-bit precision. This operation introduces deterministic quantization error without modifying the backbone parameters. We use 8-bit velocity quantization in the main experiments, as it achieves the best average performance in our ablation.

\paragraph{Weight Quantization.}
We apply per-output-channel quantization to all multi-dimensional weight tensors in the frozen reference backbone, while excluding normalization parameters and biases. We evaluate both 4-bit and 8-bit variants. Unlike velocity quantization, this degradation affects intermediate representations throughout the denoising network.

\paragraph{Gaussian Velocity Perturbation.}
We add zero-mean Gaussian noise to the reference velocity output, with the noise scale set relative to the empirical standard deviation of each channel. We evaluate relative noise levels of $0.01$ and $0.1$. This construction provides a stochastic alternative to the deterministic perturbations introduced by quantization.

\paragraph{Performance of the Degraded References}
\label{app:degraded-reference-rewards}

\begin{table*}[!h]
\centering
\setlength{\tabcolsep}{2mm}
\renewcommand{\arraystretch}{1.0}
\scriptsize
{
\begin{tabular}{lccccccc}
\toprule
\textbf{Reference} &
\textbf{GenEval} &
\textbf{OCR} &
\textbf{PickScore} &
\textbf{ClipScore} &
\textbf{HPSv2.1} &
\textbf{Aesthetic} &
\textbf{ImageReward} \\
\midrule
Original reference & 0.2529 & 0.1377 & 20.519 & 0.2384 & 0.2052 & 5.1614 & -0.5471 \\
\midrule
4-bit velocity quantization & 0.1939$\downarrow$ & 0.1292$\downarrow$ & 20.540$\uparrow$ & 0.2345$\downarrow$ & 0.2118$\uparrow$ & 5.1908$\uparrow$ & -0.5256$\uparrow$ \\
\rowcolor{lightbluecell}
\textbf{8-bit velocity quantization} & 0.2501$\downarrow$ & 0.1327$\downarrow$ & 20.550$\uparrow$ & 0.2383$\downarrow$ & 0.2064$\uparrow$ & 5.1162$\downarrow$ & -0.5457$\uparrow$ \\
4-bit weight quantization & 0.0081$\downarrow$ & 0.0112$\downarrow$ & 18.312$\downarrow$ & 0.1644$\downarrow$ & 0.0964$\downarrow$ & 4.0362$\downarrow$ & -2.2245$\downarrow$\\
8-bit weight quantization & 0.2371$\downarrow$ & 0.1295$\downarrow$ & 20.519$=$ & 0.2383$\downarrow$ & 0.2049$\downarrow$ & 5.1252$\downarrow$ & -0.5881$\downarrow$ \\
Gaussian perturbation ($\sigma=0.01$) & 0.2309$\downarrow$ & 0.1357$\downarrow$ & 20.550$\uparrow$ & 0.2397$\uparrow$ & 0.2059$\uparrow$ & 5.1288$\downarrow$ & -0.5689$\downarrow$ \\
Gaussian perturbation ($\sigma=0.1$) & 0.2514$\downarrow$ & 0.1255$\downarrow$ & 20.574$\uparrow$ & 0.2394$\uparrow$ & 0.2072$\uparrow$ & 5.1189$\downarrow$ & -0.5343$\uparrow$ \\
\bottomrule
\end{tabular}
}
\caption{Performance of the original and degraded references across all evaluated reward metrics.}
\label{tab:degraded-reference-performance}
\end{table*}

Table~\ref{tab:degraded-reference-performance} reports the reward metrics of the original and degraded references. All degradation variants reduce GenEval and OCR performance, indicating weaker alignment with these two rule-based objectives. The changes in the preference-based metrics are less consistent. Several variants slightly improve PickScore, ClipScore, HPSv2.1, or ImageReward while reducing other scores. Reference degradation therefore does not uniformly lower every reward. Instead, it changes the reward profile and increases the contrast with the task-specific teachers.

The extent of these changes depends on the degradation mechanism. 4-bit weight quantization substantially reduces every metric and may disrupt the generative structure of the reference. 8-bit velocity quantization produces a milder change. It reduces GenEval, OCR ClipScore and Aesthetic while preserving the remaining metrics. This balance provides a meaningful teacher-reference contrast without the broad performance degradation caused by aggressive weight quantization and motivates its use in our main experiments.

Gaussian perturbations inject unstructured random noise into the velocity field and may disrupt the underlying generation structure. Quantization instead introduces bounded and deterministic discretization errors that better preserve the structure of the reference predictions.

This observation also connects to the theoretical assumption underlying Proposition~\ref{thm:reference-amplification}. The reward-alignment model posits that teacher and reference distributions belong to the same parametric family $p_\rho$. This assumption holds when the reference retains the generative structure of the teacher. Mild degradation such as 8-bit velocity quantization preserves this structure and keeps the reference within the assumed family, enabling the reward-monotonicity result to hold in practice. Aggressive degradation such as 4-bit weight quantization, by contrast, may push the reference outside this family by disrupting the model's internal representations and reducing the effective extrapolation strength.

\section{Additional Experimental Results}
\label{app:additional-results}

\subsection{Complete Single-Teacher Results}
\label{app:full-single-teacher-results}

Table~\ref{tab:single_results} reports the complete results for the three independently distilled students. Each student is optimized using one task-specific teacher. 

\begin{table*}[!h]
\centering
\setlength{\tabcolsep}{1mm}
{
\scriptsize
\renewcommand{\arraystretch}{0.95}
\begin{tabular}{c l|ccccc|cc}
\toprule
& & \multicolumn{5}{c|}{\textbf{In-Domain Reward}} 
  & \multicolumn{2}{c}{\textbf{OOD Reward}} \\
\cmidrule(lr){3-7} \cmidrule(lr){8-9}
& \multirow{2}{*}[2.8ex]{\textbf{Model}}
& \textbf{GenEval} 
& \textbf{OCR} 
& \textbf{PickScore} 
& \textbf{ClipScore} 
& \textbf{HPSv2.1} 
& \textbf{Aesthetic} 
& \textbf{ImgReward} \\
\midrule
\multirow{2}{*}{\rotatebox[origin=c]{90}{Base}}
& \textbf{\texttt{SD3.5-M$\,$(w/o CFG)}}
& 0.2529 & 0.1377 & 20.519 & 0.2384 & 0.2052 & 5.161 & -0.5471  \\
& \textbf{\texttt{SD3.5-M}}
& 0.6273 & 0.5079 & 22.331 & 0.2837 & 0.2795 & 5.396 & 0.8324\\
\midrule
% \multicolumn{7}{l}{\textit{Single-Task RL}} \\
\multirow{5}{*}{\rotatebox[origin=c]{90}{GenEval}}
& \textbf{\texttt{GenEval Teacher}}
& \cellcolor{lightgraycell}{0.9470} & 0.6286 & 20.084 & 0.2870 & 0.2644 & 5.246 & 0.8976 \\
\cmidrule(lr){2-9}
& \textbf{\texttt{Flow-OPD}}
& \cellcolor{lightgraycell}{0.9395$_{\textcolor{negred}{-.0075}}$} & 0.6077& 22.161 & 0.2881 & 0.2673& 5.270& 0.9507\\
& \textbf{\texttt{DiffusionOPD}}
& \cellcolor{lightgraycell}{0.9607$_{\textcolor{posgreen}{+.0137}}$} & 0.4641 & 22.126 & 0.2794 & 0.2447 & 5.169 & 0.6209\\
& \textbf{\texttt{Ours$\,$(w/o DeRef)}}
& \cellcolor{lightgraycell}{\underline{0.9668}$_{\textcolor{posgreen}{+.0198}}$} & 0.5399 & 22.113 & 0.2797& 0.2460& 5.163& 0.6312\\
& \textbf{\texttt{Ours}}
& \cellcolor{lightgraycell}{\textbf{0.9681}$_{\textcolor{posgreen}{+.0211}}$} & 0.4748 & 22.209& 0.2766& 0.2446& 5.179 & 0.5873\\
\midrule
\multirow{5}{*}{\rotatebox[origin=c]{90}{OCR}}
& \textbf{\texttt{OCR Teacher}}
& 0.6562 & \cellcolor{lightgraycell}{0.9239} & 22.225 & 0.2919 & 0.2720 & 5.266 & 0.8881 \\
\cmidrule(lr){2-9}
& \textbf{\texttt{Flow-OPD}}
& 0.6556 & \cellcolor{lightgraycell}{0.9279$_{\textcolor{posgreen}{+.0040}}$} & 22.227 & 0.2905 & 0.2719 & 5.267& 0.9029\\
& \textbf{\texttt{DiffusionOPD}}
& 0.6533 & \cellcolor{lightgraycell}{0.9246$_{\textcolor{posgreen}{+.0007}}$} & 22.199& 0.2905& 0.2709& 5.264& 0.8721\\
& \textbf{\texttt{Ours$\,$(w/o DeRef)}}
& 0.6610 & \cellcolor{lightgraycell}{\underline{0.9322}$_{\textcolor{posgreen}{+.0083}}$} & 22.162& 0.2906& 0.2719& 5.276& 0.8946\\
& \textbf{\texttt{Ours}}
& 0.6638& \cellcolor{lightgraycell}{\textbf{0.9364}$_{\textcolor{posgreen}{+.0125}}$} & 22.158& 0.2911& 0.2719 & 5.279&0.9007 \\
\midrule
\multirow{5}{*}{\rotatebox[origin=c]{90}{Aes.}}
& \textbf{\texttt{Aes Teacher}}
& 0.4935 & 0.5014 & \cellcolor{lightgraycell}{\underline{24.034}} & \cellcolor{lightgraycell}{\textbf{0.2963}} & \cellcolor{lightgraycell}{\underline{0.3460}} & 6.232 & 1.5071 \\
\cmidrule(lr){2-9}
& \textbf{\texttt{Flow-OPD}}
& 0.4758& 0.4913& \cellcolor{lightgraycell}{24.006$_{\textcolor{negred}{-.0280}}$} & \cellcolor{lightgraycell}{\underline{0.2957}$_{\textcolor{negred}{-.0006}}$} & \cellcolor{lightgraycell}{0.3445$_{\textcolor{negred}{-.0015}}$} & 6.216& 1.5027\\
& \textbf{\texttt{DiffusionOPD}}
& 0.4738& 0.4914 & \cellcolor{lightgraycell}{24.008$_{\textcolor{negred}{-.0260}}$} & \cellcolor{lightgraycell}{0.2955$_{\textcolor{negred}{-.0008}}$} & \cellcolor{lightgraycell}{0.3454$_{\textcolor{negred}{-.0006}}$} & 6.233& 1.5030\\
& \textbf{\texttt{Ours$\,$(w/o DeRef)}}
& 0.4818& 0.4972& \cellcolor{lightgraycell}{\underline{24.034}$_{\textcolor{posgreen}{+.0000}}$} & \cellcolor{lightgraycell}{0.2946$_{\textcolor{negred}{-.0017}}$} & \cellcolor{lightgraycell}{\underline{0.3492}$_{\textcolor{posgreen}{+.0032}}$} & \underline{6.297}& \underline{1.5184}\\
& \textbf{\texttt{Ours}}
& 0.4739 & 0.4730& \cellcolor{lightgraycell}{\textbf{24.037}$_{\textcolor{posgreen}{+.0030}}$} & \cellcolor{lightgraycell}{0.2955$_{\textcolor{negred}{-.0008}}$} & \cellcolor{lightgraycell}{\textbf{0.3497}$_{\textcolor{posgreen}{+.0037}}$} & \textbf{6.301}& \textbf{1.5199}\\
\bottomrule
\end{tabular}
}
\caption{Complete results of single-task OPD methods. \textbf{Bold}: best; \underline{Underline}: second best; \colorbox{lightgraycell}{Gray-colored}: In-Domain reward; \textbf{\texttt{DeRef}}: degraded reference. Subscripts show absolute changes from the corresponding teacher.}
\label{tab:single_results}
\end{table*}

For \textbf{\texttt{GenEval Teacher}} distillation, Our \textbf{\texttt{DreOPD}} improves the teacher score from $0.9470$ to $0.9681$ and achieves the best GenEval result among the compared methods. The degraded reference provides a further gain over the non-degraded variant from $0.9668$ to $0.9681$. This stronger specialization does not improve every secondary metric. In particular, OCR and several perceptual scores decrease relative to the GenEval teacher. The result indicates that extrapolation strengthens the target capability while retaining the trade-offs associated with single-task optimization.

For \textbf{\texttt{OCR Teacher}} distillation, Our \textbf{\texttt{DreOPD}} improves OCR from $0.9239$ to $0.9364$. It also slightly improves GenEval, Aesthetic Score, and ImageReward over the OCR teacher while keeping the remaining preference metrics close to their teacher values. Compared with the non-degraded variant, the degraded reference raises OCR from $0.9322$ to $0.9364$ and produces small gains on most secondary metrics. This setting shows the clearest improvement from reference degradation without a substantial reduction in the other evaluated capabilities.

For \textbf{\texttt{Aes Teacher}} distillation, Our \textbf{\texttt{DreOPD}} improves PickScore from $24.034$ to $24.037$ and HPSv2.1 from $0.3460$ to $0.3497$, while ClipScore decreases slightly from $0.2963$ to $0.2955$. The same student also improves the out-of-domain Aesthetic Score from $6.232$ to $6.301$ and ImageReward from $1.5071$ to $1.5199$. The degraded reference improves all five perceptual metrics over the non-degraded variant, although GenEval and OCR decrease. 

Across all three settings, \textbf{\texttt{DreOPD}} achieves the largest improvement on the task-specific metrics. Some off-task metrics decrease as the student becomes more specialized, although these changes are generally moderate when considered across the full set of metrics. The overall results therefore indicate a favorable balance between improving the target capability and retaining the remaining capabilities.

\subsection{Comparison with Training-Free Model Merging}
\label{app:expo-results}
We further compare \textbf{\texttt{DreOPD}} with training-free model merging baselines. Given three task-specific teachers, we evaluate standard weight averaging, which directly merges their parameters without additional training. We also consider ExPO weight extrapolation~\citep{zheng2025expo}, where the averaged teacher model is extrapolated with $\alpha \in \{0.25, 0.5, 0.75, 1.0\}$ following the model extrapolation formulation. These training-free methods are simple and efficient, but they combine teachers only in parameter space and do not optimize the merged model on generation trajectories. As a result, weight averaging may dilute specialized capabilities, while weight extrapolation may move along directions that are not aligned with generation quality. In contrast, \textbf{\texttt{DreOPD}} performs on-policy distillation with extrapolative velocity targets. As shown in Table~\ref{tab:expo_results}, this leads to stronger and more balanced performance across evaluation metrics than training-free merging baselines.

\begin{table*}[!h]
\centering
\setlength{\tabcolsep}{1.3mm}
{
\small
\renewcommand{\arraystretch}{0.95}
\begin{tabular}{l|ccccc|cc}
\toprule
\textbf{Model} & \textbf{GenEval} & \textbf{OCR} & \textbf{PickScore} & \textbf{ClipScore} & \textbf{HPSv2.1} & \textbf{Aesthetic} & \textbf{ImgReward}\\
\midrule
\textbf{\texttt{SD3.5-M}}
& 0.6273 & 0.5079 & 22.331 & 0.2837 & 0.2795 & 5.396 & 0.8324 \\
\midrule
\multicolumn{7}{l}{\textit{Teachers}} \\
\textbf{\texttt{GenEval Teacher}}
& \cellcolor{lightgraycell}{0.9470} & 0.6286 & 20.084 & 0.2870 & 0.2644 & 5.246 & 0.8976\\
\textbf{\texttt{OCR Teacher}}
& 0.6562 & \cellcolor{lightgraycell}{0.9239} & 22.225 & 0.2919 & 0.2720 & 5.266 & 0.8881\\
\textbf{\texttt{Aes Teacher}}
& 0.4935 & 0.5014 & \cellcolor{lightgraycell}{\underline{24.034}} & \cellcolor{lightgraycell}{\underline{0.2963}} & \cellcolor{lightgraycell}{\underline{0.3460}} & 6.232 & 1.5071\\
\midrule
\multicolumn{7}{l}{\textit{Model Merge}} \\
\textbf{\texttt{Weight Average}}
& 0.8019 & 0.6952 & 22.761 & 0.2915 & 0.2965 & 5.439 & 1.1509 \\
\textbf{\texttt{ExPO $\alpha=0.25$}}
& 0.8235 & 0.7304 & 22.805 & 0.2923 & 0.2983 & 5.440 & 1.2034 \\
\textbf{\texttt{ExPO $\alpha=0.5$}}
& 0.8465 & 0.7670 & 22.828 & 0.2941 & 0.2999 & 5.449 & 1.2379 \\
\textbf{\texttt{ExPO $\alpha=0.75$}}
& 0.8569 & 0.7901 & 22.850 & 0.2951 & 0.3006 & 5.439 & 1.2654 \\
\textbf{\texttt{ExPO $\alpha=1.0$}}
& 0.8685 & 0.7994 & 22.856 & 0.2958 & 0.3014 & 5.443 & 1.2859 \\
\midrule
\rowcolor{lightbluecell}
\textbf{\texttt{Ours$\,$(w/o DeRef)}}
& \underline{0.9668} & 0.9281 & 24.032 & 0.2956 & \textbf{0.3487} & \underline{6.239} & \underline{1.5172} \\
\rowcolor{lightbluecell}
\textbf{\texttt{Ours}}
& \textbf{0.9681} & \underline{0.9362} & \textbf{24.035} & 0.2959 & \textbf{0.3487} & \textbf{6.292} & \textbf{1.5245}\\
\bottomrule
\end{tabular}
}
\caption{Comparison with training-free model merging methods using three task-specific teachers. \textbf{Bold}: best; \underline{Underline}: second best; \textbf{\texttt{DeRef}}: degraded reference; \colorbox{lightgraycell}{Gray-colored}: In-Domain reward; \colorbox{lightbluecell}{Blue-colored}: Ours.}
\label{tab:expo_results}
\end{table*}

\section{Prompts for Text-to-Image Generation}
Below we list the prompts for text-to-image generation in Figure~\ref{fig:teaser}.
\begin{itemize}
    \item A gray tabby cat sitting on a matte concrete floor against a plain pale gray studio wall.
    \item A single ripe strawberry hanging from its plant, soft cloudy daylight, plain softly blurred dark green leaf background, individual seeds and glossy red skin clearly resolved, tiny water droplets on the surface, 100mm macro, realistic botanical photograph.
    \item A single wooden pier extending into a calm lake at dawn, glassy water, mist over the surface, realistic landscape photograph.
    \item A reading-nook flat-lay on a plain matte pale-gray tabletop: an open hardcover book in the center, a folded charcoal-gray wool blanket in the upper left, a matte-white porcelain cup of tea in the upper right, and a small pair of round reading glasses in the lower center, soft cool north-window light, uncluttered pale background, top-down 50mm f/4, realistic photograph.
    \item A bowl of noodles on a wooden table.
    \item An anime girl on the seaside bathed in morning light.
    \item A deer standing in a foggy forest.
    \item A tea ceremony still-life on a plain matte pale-gray slate tray: a small cast-iron teapot in the center, a small ceramic bowl of loose green tea leaves in the upper left, three tiny porcelain tea cups arranged in an arc below the teapot, and a bamboo whisk in the upper right, soft cool diffused daylight, uncluttered pale background, top-down 50mm f/4, realistic photograph.
    \item A great horned owl perched on a bare weathered branch against a smooth deep-blue twilight sky.
    \item A weathered enamel sign reading "FRESH BREAD DAILY" hanging above a plain oak bakery counter, uncluttered warm-gray plaster wall background, tiny paint chips on the enamel.
    \item A cinematic close-up portrait of an elderly fisherman with a weathered face and kind eyes, soft diffused window light, shallow depth of field.
    \item A watercolor painting of cherry blossom trees beside a quiet river.
\end{itemize}

\section{Additional Qualitative Results}
\label{app:qualitative-results}

Figure~\ref{fig:app-qualitative-aesthetics} to Figure~\ref{fig:app-qualitative-ocr} presents additional examples over prompt-following, text-rendering and aesthetic generation.

\begin{figure*}[!t]
    \centering
    \includegraphics[width=1\textwidth]{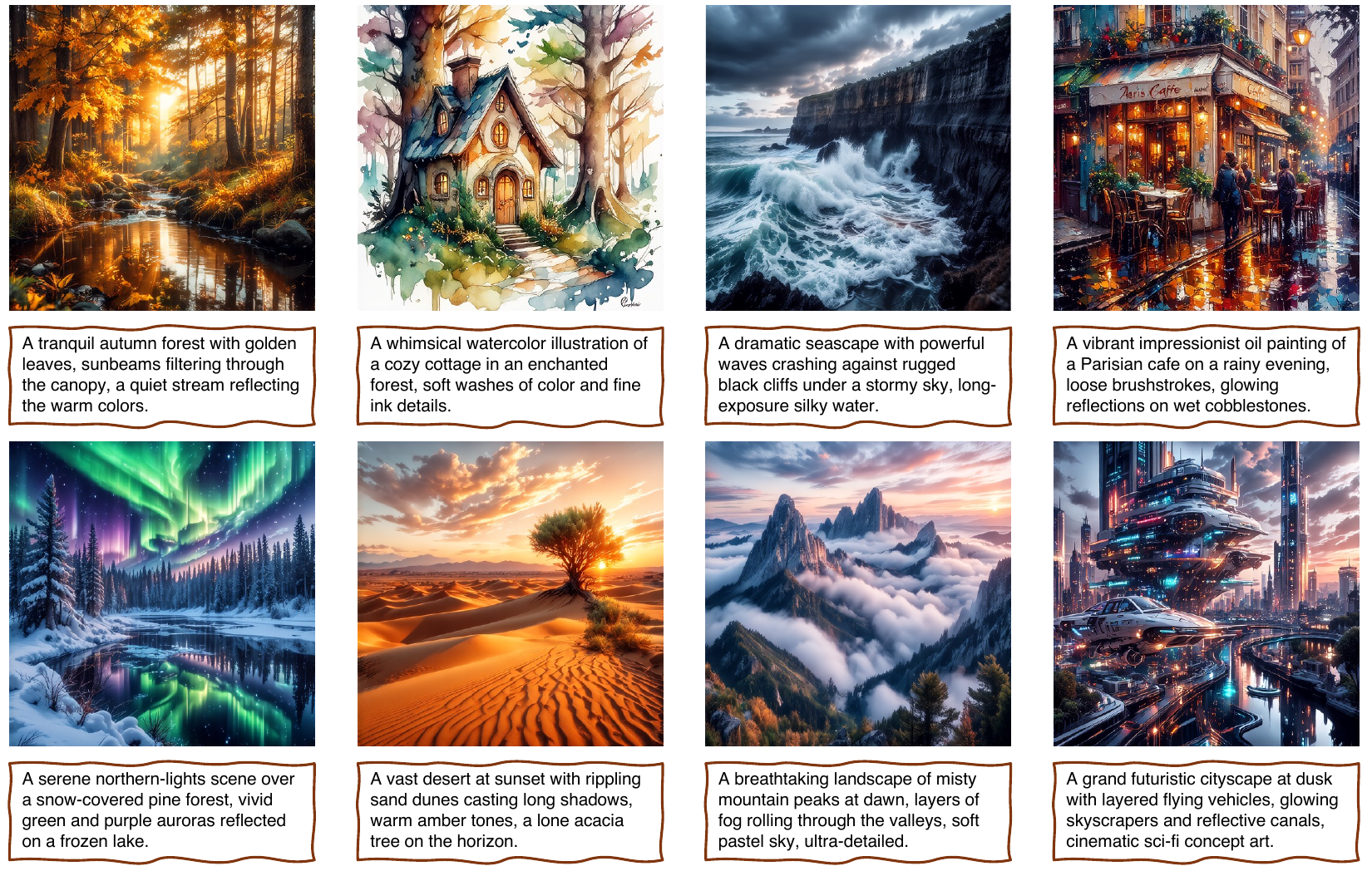}
    \caption{Additional qualitative results on aesthetic generation. The examples cover natural scenes, artistic styles, and complex compositions.}
    \label{fig:app-qualitative-aesthetics}
\end{figure*}

\begin{figure*}[!t]
    \centering
    \includegraphics[width=1\textwidth]{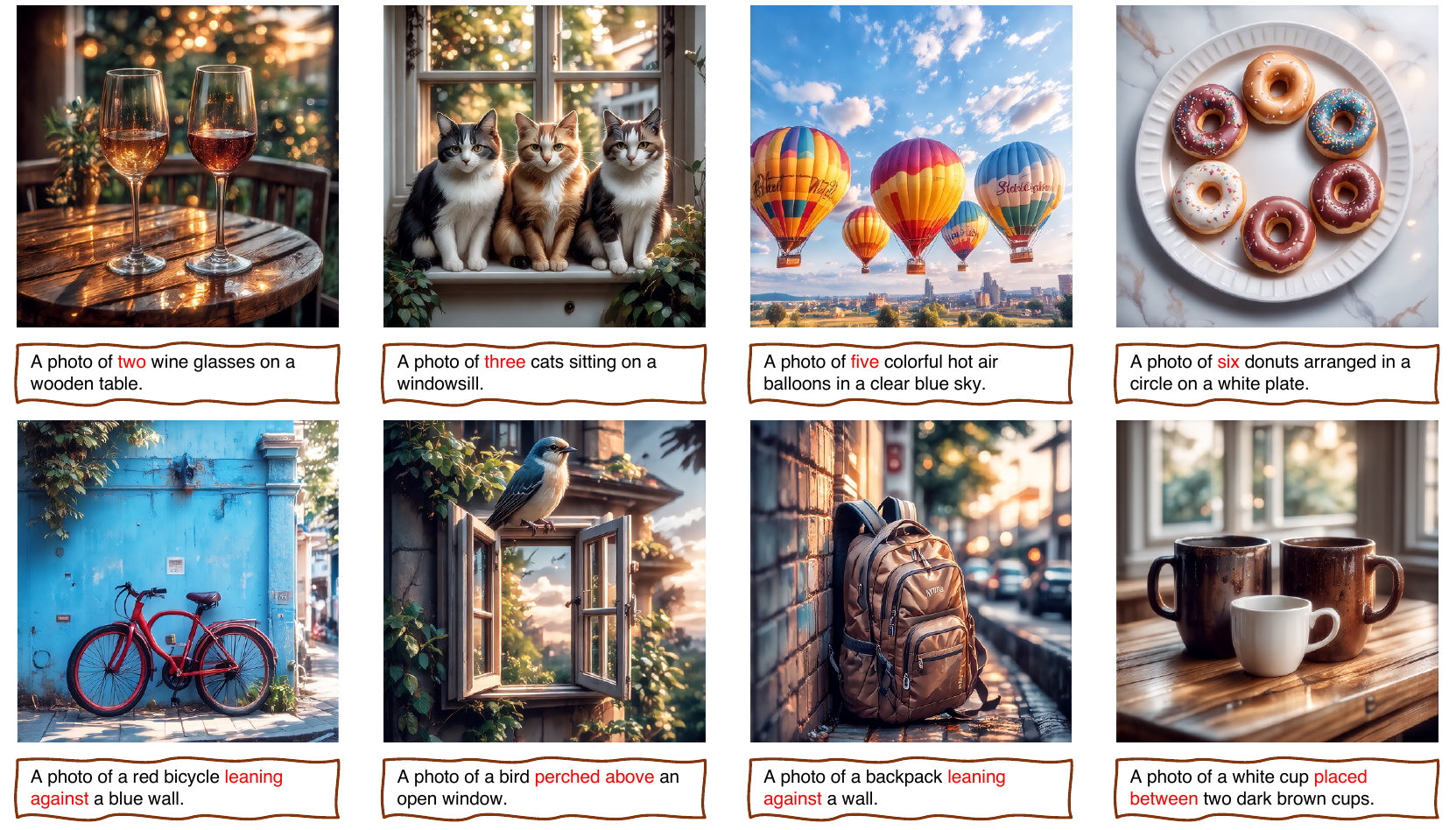}
    \caption{Additional qualitative results on compositional prompt following. The examples cover counting, color recognition, spatial relations, and attribute binding.}
    \label{fig:app-qualitative-geneval}
\end{figure*}

\begin{figure*}[!t]
    \centering
    \includegraphics[width=1\textwidth]{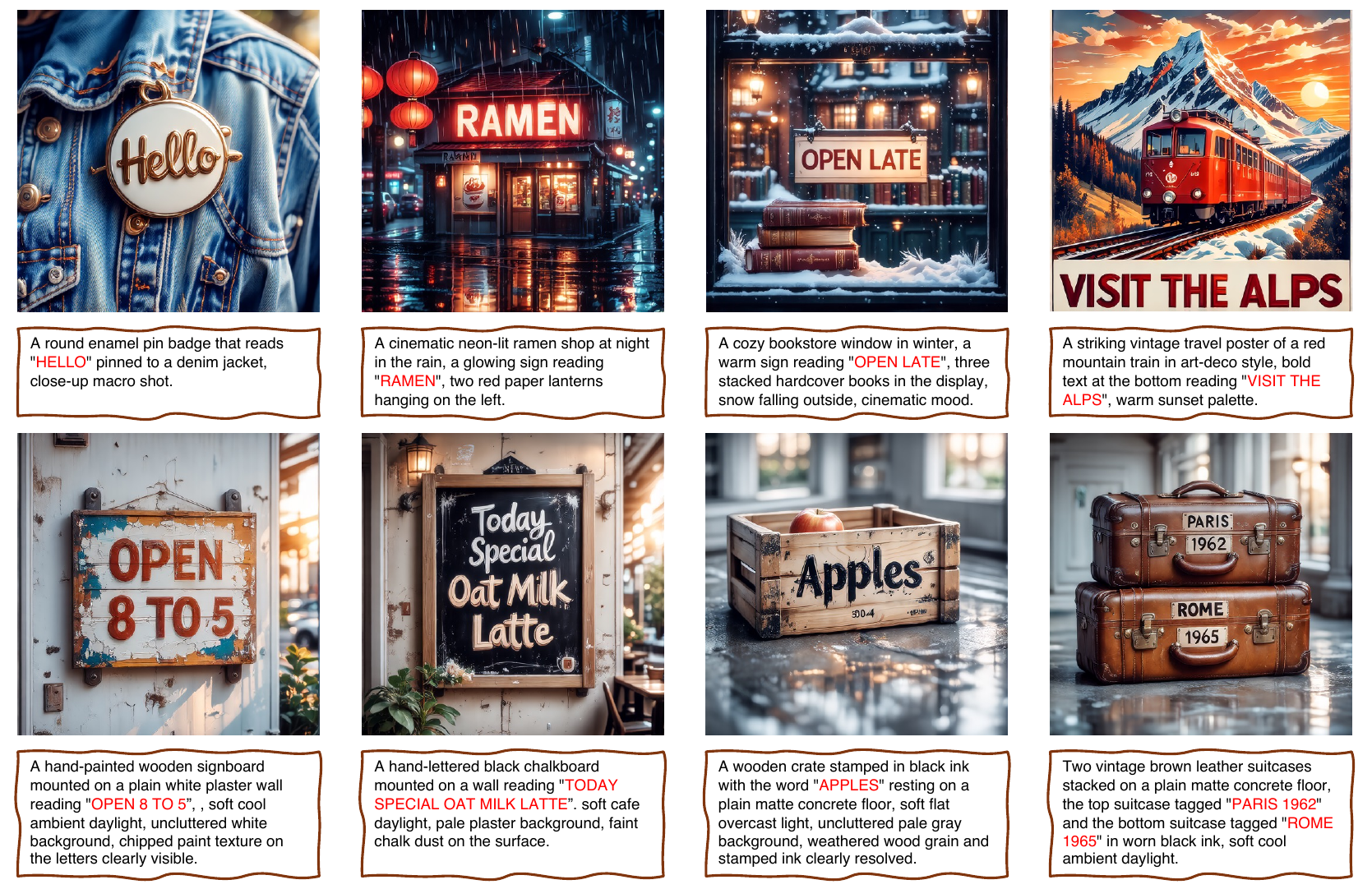}
    \caption{Additional qualitative results on text rendering.
    The examples include short text, long text an complex backgrounds.}
    \label{fig:app-qualitative-ocr}
\end{figure*}

\section{Future Directions}
Future work may explore more adaptive ways to construct degraded references and schedule the extrapolation strength during training. It would also be valuable to extend \textbf{\texttt{DreOPD}} beyond text-to-image generation, such as video generation and controllable image editing, to further examine the generality of degraded-reference extrapolative distillation.

\end{document}

%% file: math_commands.tex
\usepackage{amsmath,amsfonts,bm}

\def\eqref#1{equation~\ref{#1}}
\def\1{\bm{1}}

\DeclareMathAlphabet{\mathsfit}{\encodingdefault}{\sfdefault}{m}{sl}
\SetMathAlphabet{\mathsfit}{bold}{\encodingdefault}{\sfdefault}{bx}{n}